%% file: ms.tex
\pdfoutput=1

\documentclass[twoside]{article}

\usepackage[accepted]{aistats2024}
\usepackage{bm}
\usepackage{xcolor}
\usepackage{amsmath}
\usepackage{amsthm}
\usepackage{amsfonts}       % blackboard math symbols
\usepackage{algorithm}
\usepackage{algorithmic}
\usepackage{subcaption}
\usepackage{graphicx}
\graphicspath{ {./img/} }

\usepackage{Sty/mcr}

\newtheorem{lemma}{Lemma}

\newtheorem{example}{Example}

\DeclareMathOperator*{\argmin}{arg\,min}

\begin{document}

\runningtitle{Statistical Inference for the DTW Distance, with Application to Abnormal Time-Series Detection}


\twocolumn[

\aistatstitle{Statistical Inference for the Dynamic Time Warping Distance, \\ with Application to Abnormal Time-Series Detection}

\aistatsauthor{ Vo Nguyen Le Duy \And Ichiro Takeuchi }

\aistatsaddress{ RIKEN \And  Nagoya University/RIKEN} ]

\begin{abstract}
\input{abst}
\end{abstract}

\input{sec1}

\input{sec2}

\input{sec3}

\input{sec4}

\input{sec5}

\input{sec6}

\bibliographystyle{abbrv}
\bibliography{ref}

\newpage
\onecolumn
%\aistatstitle{Instructions for Paper Submissions to AISTATS 2021: \\
%Supplementary Materials}

\input{appendix}

\end{document}

%% file: abst.tex
We study statistical inference on the similarity/distance between two time-series under uncertain environment by considering a hypothesis test on the distance obtained from Dynamic Time Warping (DTW) algorithm.
The sampling distribution of the DTW distance is difficult to derive because it is obtained based on the solution of the DTW algorithm, which is complicated.
To circumvent this difficulty, we propose to employ the \emph{conditional selective inference} framework, which enables us to derive a \emph{valid} inference method on the DTW distance.
%
%Besides, we also develop a novel computational method to compute the conditional sampling distribution.
%
To our knowledge, this is the first method that can provide a valid $p$-value to quantify the statistical significance of the DTW distance, which is helpful for high-stake decision making such as abnormal time-series detection problems.
%
%We evaluate the performance of the proposed inference method on both synthetic and real-world datasets.

%% file: sec1.tex
\section{Introduction} \label{sec:intro}

Abnormal time-series detection is a crucial task in various fields. 
A fundamental method %for identifying abnormal time-series 
is to compare a new query time-series to a reference (normal) time-series. 
To do this, it is often necessary to align the two time-series and then measure the distance between them.
If the distance exceeds a pre-determined threshold, the query time-series is considered abnormal.
Aligning two time-series involves computing the optimal pairwise correspondence between their elements while preserving the alignment orderings.
The Dynamic Time Warping (DTW) \cite{sakoe1978dynamic} is a standard algorithm for finding the optimal alignment between the two %given 
time-series.

In order to control the balance between two types of errors in abnormality detection, i.e., false positives (errors in which normal time series are incorrectly identified as abnormal) and false negatives (errors in which abnormal time series are falsely determined as normal), it is necessary to consider the statistical reliability of the DTW distance.
Our goal is to develop a statistical inference for the DTW distance, in the form of $p$-value or confidence interval, to control the false positive rate (FPR).
In other words, if we repeat the abnormal time-series detections many times, the probability of obtaining incorrect abnormal time-series can be controlled under a significance level $\alpha$ (e.g., 0.05).

However, this task is challenging because the sampling distribution of the DTW distance is too complex to derive.
%, i.e., it is difficult to analyze how the uncertainty in the observed time-series affects the uncertainty in the DTW distance.
%
Our key idea to circumvent this difficulty is to employ the \emph{conditional Selective Inference (SI)} literature~\cite{lee2016exact}.
The basic concept of conditional SI is to make an inference conditional on a \emph{selection event}.
In this paper, we interpret the optimization problem of selecting the optimal alignment between the two time-series as the selection event.
By conditioning on the optimal alignment, the sampling distribution of the DTW distance can be derived which is subsequently used to conduct the statistical inference. 
We would like to note that we do \emph{not} introduce a new anomaly detection method in this study.
Instead, we introduce a novel post-inference method on the results obtained after the abnormal time-series detection is performed.

\begin{example}
To see the importance of the proposed method, we consider the results in Table \ref{tbl:example_intro}. 
We generated a query time-series and a reference time-series that were both normal.
% (i.e., both time-series have the same underlying signal).
%
Then, we calculated the DTW distance and conducted abnormal time-series detection.
We compared our method with three other methods: no inference (comparing the distance with the threshold without inference), naive statistical inference %(classical $z$-test) 
and data splitting.
The experiment was repeated $N$ times and the FPR results are shown in Table \ref{tbl:example_intro}.
With the proposed method, we were able to control the FPR under $\alpha$ = 0.05, which the competitors were unable to achieve.
In the proposed method, even if the threshold is arbitrarily determined, it is possible to adjust it in a way that ensures the FPR is smaller than %the desired  
$\alpha$.
\end{example}

\begin{table}[!t]
\renewcommand{\arraystretch}{1.2}
\centering
\caption{The importance of the proposed method lies in its ability to control the FPR (type-I error rate). When statistical inference was either not performed or conducted improperly, we failed to control the FPR. However, with the proposed method, a \emph{valid} statistical inference was performed, leading to successful control of the FPR at a significance level $\alpha$ = 0.05.}
\vspace{3pt}
\begin{tabular}{ |l|c|c| } 
  \hline
  & $N = 1200$ & $N = 2400$ \\
  \hline
   \textbf{No Inference} & FPR = 0.87 & FPR = 0.76  \\
   \hline
   \textbf{Naive Inference} & 0.80 & 0.78  \\
   \hline
   \textbf{Data Splitting} & 0.12 & 0.11  \\
   \hline
   \textbf{Proposed Method} & \textbf{0.04} & \textbf{0.05}  \\
  \hline
\end{tabular}
\label{tbl:example_intro}
\vspace{-12pt}
\end{table}

\textbf{Contribution.}
The main contributions in this study are two-fold.
The first contribution is that we derive a conditional sampling distribution of the DTW distance in a tractable form inspired by the conditional SI literature. 
This task can be done by conditioning on the optimal alignment between the two time-series.
%
%Thus, based on the distribution, we can conduct a valid inference on the DTW distance.
%
%Whereas the \emph{un}conditional sampling distribution of the DTW distance is too complex to characterize, we show that the sampling distribution conditional on the optimal alignment can be characterized in a tractable form. 
%
The second contribution is to develop a computational method to compute the conditional sampling distribution by introducing a non-trivial technique called \emph{parametric DTW method}. 
These two contributions enable us to detect abnormal query time-series with valid statistical significance measures such as $p$-values or confidence intervals. 
To our knowledge, this is the first valid statistical test for the DTW distance, which is essential for controlling the risk of high-stakes decision making in signal processing.
%
%\red{Table \ref{} shows the importance of the proposed method.}
%Figure \ref{fig:intro} shows an illustrative example of the proposed $p$-value in an abnormal heart beat detection problem.
%
%Our implementation is provided in the supplementary material.
%

\textbf{Related work.}
Anomaly detection in time series is a problem in which the goal is to identify anomalous points \emph{within} the time-series that can indicate potential anomalies.
There is a vast body of literature on methods for this problem, and a reference can be found in \cite{aggarwal2017outlier}.
However, this paper focuses on a different problem: abnormal time-series detection, in which the goal is to identify if the \emph{entire} query time-series is abnormal.
The fundamental approach in the latter problem involves computing the distance between a new query time-series and a reference time-series, and %subsequently 
comparing the resulting distance against a %pre-determined 
given threshold to determine if the new query time-series is abnormal.

The DTW distance is commonly used for quantifying the similarity/distance between two time-series~\cite{sakoe1978dynamic, keogh2001derivative, muller2007dynamic, cuturi2017soft}.
However, due to the complex discrete nature of the DTW algorithm, it is difficult to quantify the uncertainty of the DTW distance.
Therefore, to our knowledge, there are neither valid methods nor asymptotic approximation methods for the statistical inference on the DTW distance.
Due to the lack of valid statistical inference method, when decision making is conducted based on DTW distance, it is difficult to properly control the risk of the incorrect decision. 

In recent years, conditional SI has emerged as a promising approach for evaluating the statistical reliability of data-driven hypotheses. 
It was first introduced as a statistical inference tool for the features selected by Lasso \cite{lee2016exact}.
The concept behind conditional SI is to make inference based on the sampling distribution of the test statistic conditional on a selection event.
This approach allows us to derive the exact sampling distribution of the test statistic.
Conditional SI has also been applied to various problems \cite{loftus2015selective, choi2017selecting, tian2018selective, yang2016selective, tibshirani2016exact, fithian2014optimal, loftus2014significance, panigrahi2016bayesian,  sugiyama2021more, hyun2018post, duy2021more, le2021parametric, duy2021exact, sugiyama2021valid, chen2019valid, tsukurimichi2021conditional, tanizaki2020computing, duy2020computing, duy2020quantifying}. 
%\footnote{More details on the relation between the proposed method and conditional SI literature are presented in \S \ref{sec:conditional_si_dtw_distance}.}. 
%
However, no study to date can utilize conditional SI to provide a statistical inference on the DTW distance.
%
%Figure \ref{fig:connection_si_lasso_dtw} demonstrate the connection between the proposed method and the seminal conditional SI work.

The most closely related work (and the motivation for this study) is \cite{duy2021exact}, where the authors introduce SI for computing a confidence interval for the Wasserstein distance (WD). 
The idea of \cite{duy2021exact} is to consider the distribution of the WD conditional on the transportation plan.
%, which enables us to conduce a valid statistical inference for the WD.
%
Their method relies on the fact that the WD is defined as the solution of a linear program (LP), and specific properties of an LP can be utilized to achieve the goal.
However, it is not the case of the DTW distance because it is defined as the solution of a combinatorial optimization solved by Dynamic Programming, which is more complicated. 
Therefore, the method in \cite{duy2021exact} is \emph{not} applicable in the case of the DTW distance.

%% file: sec2.tex
\section{Problem Statement} \label{sec:problem_statement}
Let us consider a query time-series $\bm X$ and a normal reference time-series $\bm Y$ represented as vectors corrupted with Gaussian noise and denote them as
\begin{subequations}
\begin{align}
	\hspace{-3pt} \bm X &= (x_1, ..., x_n)^\top = \bm \mu_{\bm X}  + \bm \veps_{\bm X},  ~ \bm \veps_{\bm X} \sim \NN(\bm 0, \Sigma_{\bm X}), \label{eq:random_X} \\ 
	\hspace{-3pt} \bm Y &= (y_1, ..., y_m)^\top = \bm \mu_{\bm Y}  + \bm \veps_{\bm Y},  ~ \bm \veps_{\bm Y} \sim \NN(\bm 0, \Sigma_{\bm Y}) \label{eq:random_Y},
\end{align}
\end{subequations}
where $n$ and $m$ are the lengths of time-series,
$\bm \mu_{\bm X}$ and $\bm \mu_{\bm Y}$ are the signal vectors, 
$\bm \veps_{\bm X}$ and $\bm \veps_{\bm Y}$ are Gaussian noise vectors with covariances matrices $\Sigma_{\bm X}$ and $\Sigma_{\bm Y}$ are known or estimable from independent data.

% ================================

%\begin{figure}[!t]
%\centering
%\includegraphics[width=.65\linewidth]{optimal_alignment_matrix.pdf}
%\caption{
%An example of (binary) optimal alignment matrix $\hat{M}$.
%%
%The DTW method calculates an optimal match (blue) between two given time-series with certain rules: monotonicity, continuity, and matching endpoints.
%%
%After obtaining the optimal match, we can define the optimal alignment matrix.
%}
%\label{fig:optimal_alignment_matrix}
%\end{figure}

\subsection{Optimal Alignment and the DTW}
Let us denote the cost matrix of pairwise distances between the elements of $\bm X$ and $\bm Y$ as 
\begin{align} \label{eq:cost_matrix}
	C(\bm X, \bm Y) 
	& = \big[(x_i - y_j)^2 \big]_{ij} \in \RR^{n \times m}.
\end{align}
The \emph{optimal alignment matrix} between $\bm X$ and $\bm Y$ is %defined as 
\begin{align} \label{eq:optimal_alignment}
		\hat{M} = \argmin \limits_{M \in \cM_{n, m}} \big \langle M, C(\bm X, \bm Y) \big \rangle,
\end{align}
where $\cM_{n, m} \subset \{0, 1\}^{n \times m}$ is a set of (binary) alignment matrices that satisfy the monotonicity, continuity, and matching endpoints constraints,
% as illustrated in Figure \ref{fig:optimal_alignment_matrix},
%\footnote{Alignment matrix shows a path on a $n \times m$ matrix that connects the upper-left $(1, 1)$ matrix entry to the lower-right $(n, m)$ one using only $\downarrow, \rightarrow, \searrow$ moves.},
and $\langle \cdot, \cdot \rangle$ is the Frobenius inner product.
The cardinality of $\cM_{n, m}$ is known as the ${\rm delannoy}(n - 1, m - 1)$ 
which is the number of paths on a rectangular grid from (0, 0) to ($n - 1$, $m - 1$) using only single steps to south, southeast, or east direction.
A naive way to solve \eq{eq:optimal_alignment} is to enumerate all possible candidates in $\cM_{n, m}$ and obtain $\hat{M}$.
However, it is computationally impractical because the size of the set $\cM_{n, m}$ is exponentially increasing with $n$ and $m$.
The DTW is well-known as an efficient dynamic programming algorithm to obtain the solution $\hat{M}$ in \eq{eq:optimal_alignment} by using \emph{Bellman recursion}.

% ================================
\subsection{Closed-form of the DTW Distance}

After obtaining the optimal alignment matrix $\hat{M}$, the DTW distance is written in a closed form as
	\begin{align*} 
		\hat{L}(\bm X, \bm Y) 
		& = \left \langle \hat{M}, C(\bm X, \bm Y) \right \rangle  
		= \hat{M}_{\rm vec}^\top C_{\rm vec}(\bm X, \bm Y), 
	\end{align*}
	where $\hat{M}_{\rm vec} = {\rm vec} (\hat{M}) \in \RR^{n m}$,
%	$
%	~ C_{\rm vec} (\bm X, \bm Y) 
%	=
%	\left [ \Omega  (\bm X ~ \bm Y)^\top \right ]
%	\circ
%	\left [ \Omega (\bm X ~ \bm Y)^\top \right ] \in \RR^{nm}, 
%	$
	%
	\begin{align*}
	& C_{\rm vec} (\bm X, \bm Y) 
	%= {\rm vec} \Big(C(\bm X, \bm Y) \Big) 
	=
	\left [ \Omega  {\bm X \choose \bm Y} \right ]
	\circ
	\left [ \Omega  {\bm X \choose \bm Y} \right ] \in \RR^{nm}, 
	%\label{eq:C_vec} 
	\\ 
	&
	\Omega = {\rm hstack}\left ( 
	I_{n} \otimes \bm 1_{m}, - \bm 1_{n} \otimes I_{m} \right ) 
\in \RR^{n m \times (n + m)},
	%\\
	%
%	& \Omega = 
%	\begin{pmatrix}
%		\bm 1_m & \bm 0_m  & \cdots & \bm 0_m & - I_m \\ 
%		\bm 0_m & \bm 1_m  & \cdots & \bm 0_m & - I_m \\ 
%		\vdots & \vdots  & \ddots & \vdots & \vdots \\ 
%		\bm 0_m & \bm 0_m  & \cdots & \bm 1_m & - I_m
%	\end{pmatrix}
	%
%	\bordermatrix{
%		   & 1 & 2 & \cdots & n &  \cr 
%		1 & \bm 1_m & \bm 0_m  & \cdots & \bm 0_m & - I_m  \cr 
%		2 & \bm 0_m & \bm 1_m  & \cdots & \bm 0_m & - I_m  \cr 
%		  \vdots  & \vdots & \vdots  & \ddots & \vdots & \vdots  \cr 
%		n & \bm 0_m & \bm 0_m  & \cdots & \bm 1_m & - I_m
%	}
%	\in \RR^{nm \times (n + m)}, 
	\end{align*}
${\rm vec}(\cdot)$ is an operator that transforms a matrix into a vector with concatenated rows, the operator $\circ$ is element-wise product, $\rm hstack(\cdot, \cdot)$ is horizontal stack operation, $I_n \in \RR^{n \times n}$ is the identity matrix, and $\bm 1_m \in \RR^m$ is a vector of ones.
For mathematical tractability, we consider a slightly modified version of the DTW distance defined as 
\begin{align} \label{eq:distance_closed_form}
    		\hat{L}^\prime (\bm X, \bm Y)
    		= 
    		\hat{M}_{\rm vec}^\top ~ {\rm abs}
    		\left (
    		\Omega
    		{\bm X \choose \bm Y}
    		\right ),
\end{align}
where ${\rm abs} (\cdot)$ denotes the element-wise absolute operation.
Examples of vector $C_{\rm vec}(\bm X, \bm Y)$, matrix $\Omega$ and vector $\hat{M}_{\rm vec}$ are provided in Appendix \ref{appendix:examples}.

%\begin{example} \label{example:M_hat_vec_Omega_C_vec} ($C_{\rm vec}(\bm X, \bm Y)$, $\Omega$ and $\hat{M}_{\rm vec}$) Given $\bm X =  (x_1, x_2)^\top$ and $\bm Y = (y_1, y_2)^\top$, the cost matrix is  
%%
%\begin{align*}
%	C(\bm X, \bm Y) = 
%	\begin{pmatrix}
%		(x_1 - y_1)^2 & (x_1 - y_2)^2 \\ 
%		(x_2 - y_1)^2 & (x_2 - y_2)^2
%	\end{pmatrix}.
%\end{align*}
%%
%Then, we have 
%%
%\begin{align*}
%	C_{\rm vec} (\bm X, \bm Y) 
%	= 
%	\begin{pmatrix}
%		(x_1 - y_1)^2 \\ 
%		(x_1 - y_2)^2 \\ 
%		(x_2 - y_1)^2 \\
%		(x_2 - y_2)^2
%	\end{pmatrix}
%	= 
%	\Omega
%	\begin{pmatrix}
%		x_1 \\ 
%		x_2 \\ 
%		y_1 \\
%		y_ 2
%	\end{pmatrix}
%	\circ
%	\Omega
%	\begin{pmatrix}
%		x_1 \\ 
%		x_2 \\ 
%		y_1 \\
%		y_ 2
%	\end{pmatrix},
%\end{align*}
%%
%where
%$
%	\Omega = 
%	\begin{pmatrix}
%		1 & 0 & - 1 & 0 \\ 
%		1 & 0 & 0 & - 1 \\ 
%		0 & 1 & - 1 & 0 \\ 
%		0 & 1 & 0 & - 1 
%	\end{pmatrix}
%$.
%%
%Similarly, given 
%$
%\hat{M} = 
%\begin{pmatrix}
%	1 & 0 \\ 
%	0 & 1
%\end{pmatrix}
%$,
%then 
%$
%\hat{M}_{\rm vec} 
%= 
%\begin{pmatrix}
%	1 & 0 & 0 & 1
%\end{pmatrix}^\top 
%$.
%\end{example}

% ================================
\subsection{Statistical Inference}

Our goal is to test if the DTW distance between the query signal $\bm \mu_{\bm X}$ and the reference signal $\bm \mu_{\bm Y}$ is smaller or greater than a threshold.

\textbf{Null and alternative hypotheses.}
Let $\tau > 0 $ be the threshold. 
The test for abnormal time-series detection is formulated by considering following hypotheses:
\begin{align*}
		{\rm H}_0: \hat{L}^\prime (\bm \mu_{\bm X}, \bm \mu_{\bm Y}) \leq \tau 
		\quad 
		\text{vs.}
		\quad 
		{\rm H}_1: \hat{L}^\prime (\bm \mu_{\bm X}, \bm \mu_{\bm Y}) > \tau.
\end{align*}

\textbf{Test statistic.} By replacing $(\bm \mu_{\bm X}, \bm \mu_{\bm Y})$ with $(\bm X, \bm Y)$, the test statistic $T$  is defined as follows:
\begin{align} \label{eq:test_statistic_first}
	\hspace{-2pt}T
	& =
	 \hat{L}^\prime (\bm X, \bm Y) \nonumber\\ 
	&= 
	\hat{M}_{\rm vec}^\top
	~ {\rm abs}
	\left (
	 \Omega  {\bm X \choose \bm Y}
	\right ) =
%	\nonumber \\
%	&=
	~
	\hat{M}_{\rm vec}^\top
	 {\rm diag} (\hat{\bm s})\Omega {\bm X \choose \bm Y},
\end{align}
where 
$
	\hat{\bm s} = {\rm sign} 
	\left ( 
	\hat{M}_{\rm vec} 
	\circ 
	\left [ \Omega  {\bm X \choose \bm Y} \right ]
	\right ) \in \RR^{nm}
$,
${\rm sign}(\cdot)$ is the operator that
returns an element-wise indication of the sign of a number (${\rm sign} (0) = 0$),
and ${\rm diag}(\hat{\bm s})$ is the diagonal matrix whose diagonal entries are the elements of the vector $\hat{\bm s}$.
For notational simplicity, we re-write the test statistic as 
\begin{align} \label{eq:test_statistic_final}
	T = \bm \eta_{\hat{M}, \hat{\bm s}}^\top \big (\bm X^\top ~ \bm Y^\top \big)^\top, 
%	\quad 
%	\text{where}
%	\quad 
%	\bm \eta_{\hat{M}, \hat{\bm s}}  = \left(
%	\hat{M}_{\rm vec}^\top {\rm diag}(\hat{\bm s}) \Omega
%	 \right)^\top \in \RR^{n + m}
\end{align}
where 
$
\bm \eta_{\hat{M}, \hat{\bm s}}  = \left(
	\hat{M}_{\rm vec}^\top {\rm diag}(\hat{\bm s}) \Omega
	 \right)^\top \in \RR^{n + m}
$
is the direction of the test statistic.

%\red{
%We would like to note that, with the above descriptions, the two time-series are considered to be the same if 
%$
%\bm \eta_{\hat{M}, \hat{\bm s}}^\top {\bm \mu_{\bm X} \choose \bm \mu_{\bm Y}} = 0 
%$.
%%
%However, the aforementioned definition of ``the same'' might be inappropriate, especially when the two time-series have different lengths.
%%
%Although there is the case where they are the same because the shorter time-series is contained as a subsequence of the longer time-series, it might rarely happen.
%%
%Therefore, it is also important to provide the confidence interval of the DTW distance, which is more reasonable compared to the hypothesis testing setup.
%}

\textbf{Challenge of characterizing the distribution of $T$.}
For statistical inference on the DTW distance, we need to characterize the sampling distribution of the test statistic $T$ in \eq{eq:test_statistic_final}.
However, since $\bm \eta_{\hat{M}, \hat{\bm s}}$ depends on $\hat{M}$ and $\hat{\bm s}$ which are defined based on the data, characterization of the exact sampling distribution of the test statistic is intrinsically difficult.
In the next section, we introduce a novel approach to resolve the aforementioned challenge inspired by the concept of conditional SI and propose a valid \emph{selective $p$-value} to conduct valid statistical inference on the DTW distance.

%% file: sec3.tex
\section{SI for the DTW Distance} \label{sec:conditional_si_dtw_distance}

In this section, we present our first contribution.
To conduct statistical inference on the DTW distance, we employ the conditional SI framework.
Our idea is that,
%comes from the fact that, given the optimal alignment matrix $\hat{M}$, the DTW distance is written in a closed form as in \eq{eq:distance_closed_form}. 
%
by conditioning on the optimal alignment matrix $\hat{M}$ and its sign $\hat{\bm s}$, we can derive the conditional sampling distribution of the test statistic.

% ================================
\subsection{Conditional Distribution and $p$-value} \label{subsec:conditional_distribution_selective_p_value}

We consider the following conditional sampling distribution of the test statistic:
\begin{align} \label{eq:conditional_inference}
	\hspace{-6pt}
	\bm \eta_{\hat{M}, \hat{\bm s}}^\top {\bm X \choose \bm Y} \mid 
	\left \{ 
		\cA(\bm X, \bm Y) = \hat{M}^{\rm obs},
		\cS(\bm X, \bm Y) = \hat{\bm s}^{\rm obs}
	\right \}
\end{align}
%
%where $
%	\cA: (\bm X, \bm Y) \rightarrow \hat{M},  ~ 
%	\cS: (\bm X, \bm Y) \rightarrow \hat{\bm s}, ~
%	\hat{M}^{\rm obs} = \cA(\bm X^{\rm obs}, \bm Y^{\rm obs}), ~ 
%	\hat{\bm s}^{\rm obs} = \cS(\bm X^{\rm obs}, \bm Y^{\rm obs}).
%$
where we denote
$
\cA: (\bm X, \bm Y) \rightarrow \hat{M},  ~
  \cS: (\bm X, \bm Y) \rightarrow \hat{\bm s},
$
\begin{align*}
% & \cA: (\bm X, \bm Y) \rightarrow \hat{M},  \quad 
%  \cS: (\bm X, \bm Y) \rightarrow \hat{\bm s}, \\ 
 & \hat{M}^{\rm obs} = \cA(\bm X^{\rm obs}, \bm Y^{\rm obs}), \quad 
 \hat{\bm s}^{\rm obs} = \cS(\bm X^{\rm obs}, \bm Y^{\rm obs}). \nonumber
\end{align*}
%
%\begin{align*}
% & \cA: (\bm X, \bm Y) \rightarrow \hat{M},  \quad 
%  \cS: (\bm X, \bm Y) \rightarrow \hat{\bm s}, \\ 
% & \hat{M}^{\rm obs} = \cA(\bm X^{\rm obs}, \bm Y^{\rm obs}), \quad 
% \hat{\bm s}^{\rm obs} = \cS(\bm X^{\rm obs}, \bm Y^{\rm obs}). \nonumber
%\end{align*}
%
%$\hat{M}^{\rm obs} = \cA(\bm X^{\rm obs}, \bm Y^{\rm obs})$, and
%$\bm s^{\rm obs} = \cS(\bm X^{\rm obs}, \bm Y^{\rm obs})$.
%
%Here, we remind that $\bm X^{\rm obs}$ and $\bm Y^{\rm obs}$ are observations (realizations) of the random vectors $\bm X$ and $\bm Y$, respectively.
%
%
%Next, to test the statistical significance of the DTW distance, we introduce the selective $p$-value that satisfies the following sampling property:
%%
%\begin{align} \label{eq:sampling_property_selective_p}
%	\PP_{{\rm H}_0} \left (p_{\rm sel} \leq \alpha 
%	~ \mid  
%	%
%	\cA(\bm X, \bm Y) = \hat{M}^{\rm obs}, %\\ 
%	\cS(\bm X, \bm Y) = \hat{\bm s}^{\rm obs}
%	%\end{array}
%	\right ) \leq \alpha, 
%\end{align}
%%
%$\forall \alpha \in [0, 1]$, which is a crucial property for a valid $p$-value.
%
Next, we introduce the selective $p$-value defined as:
\begin{align} \label{eq:selective_p}
	p_{\rm sel}  = 
	\bP_{\rm H_0} 
	\left (
	\bm \eta_{\hat{M}, \hat{\bm s}}^\top 
	{\bm X \choose \bm Y} 
%	(\bm X ~ \bm Y)^\top
	\geq 
	\bm \eta_{\hat{M}, \hat{\bm s}}^\top 
	{\bm X^{\rm obs} \choose \bm Y^{\rm obs}}
%	(\bm X^{\rm obs} ~ \bm Y^{\rm obs})^\top
	~ \Big | ~
	\cE
	\right ),
\end{align}
where 
$\cE = 
\left \{ 
	\begin{array}{l}
	\cA(\bm X, \bm Y) = \hat{M}^{\rm obs}, 
	\cS(\bm X, \bm Y) = \hat{\bm s}^{\rm obs}, \\ 
	\cQ (\bm X, \bm Y) = \hat{\bm q}^{\rm obs}
	\end{array}  
\right \} 
$.

The $\cQ : (\bm X, \bm Y) \rightarrow \hat{\bm q}$ is the nuisance component: %defined as 
\begin{align} \label{eq:q_and_b}
	\cQ (\bm X, \bm Y) = 
	\left ( 
	I_{n+m} - 
	\bm b
	\bm \eta_{\hat{M}, \hat{\bm s}}^\top \right ) 
	(\bm X^\top ~ \bm Y^\top)^\top,
%	%
%	\quad \text{ where } \quad 
%	%
%	\bm b = \frac{\Sigma \bm \eta_{\hat{M}, \hat{\bm s}}}
%	{\bm \eta_{\hat{M}, \hat{\bm s}}^\top \Sigma \bm \eta_{\hat{M}, \hat{\bm s}}}
\end{align}
where 
$
	\bm b = \frac{\Sigma \bm \eta_{\hat{M}, \hat{\bm s}}}
	{\bm \eta_{\hat{M}, \hat{\bm s}}^\top \Sigma \bm \eta_{\hat{M}, \hat{\bm s}}}
$
and 
$
\Sigma = 
\begin{pmatrix}
	\Sigma_{\bm X} & 0 \\ 
	0 & \Sigma_{\bm Y}
\end{pmatrix}.
$
%
%There are two unknown parameters in our problem: 
%$
%\bm \eta_{\hat{M}, \hat{\bm s}}^\top {\bm \mu_{\bm X} \choose \bm \mu_{\bm Y}}
%$
%(parameter of interest) 
%and 
%$
%\cQ (\bm \mu_{\bm X}, \bm \mu_{\bm Y}) 
%$
%(nuisance parameter that is not of immediate interest).
%%
%To tractably conduct the inference on 
%$
%\bm \eta_{\hat{M}, \hat{\bm s}}^\top {\bm \mu_{\bm X} \choose \bm \mu_{\bm Y}}
%$,
%we need to eliminate the nuisance parameter from the problem by conditioning on its sufficient statistic $\cQ (\bm X, \bm Y)$.
%%
%This is one of the fundamental approaches to deal with nuisance parameter in statistics.

\begin{lemma} \label{lemma:valid_selective_p}
The selective $p$-value proposed in \eq{eq:selective_p} satisfies the property of a valid $p$-value:
\begin{align*}
	\mathbb{P}_{\rm H_0}  \Big (
	p_{\rm sel} \leq \alpha
	\Big) \leq \alpha, ~~ \forall \alpha \in [0, 1].
\end{align*} 
\end{lemma}

\begin{proof}
The proof is deferred to Appendix \ref{appx:proof_valid_selective_p}.
\end{proof}

Lemma \ref{lemma:valid_selective_p} indicates that the probability of obtaining a false positive is controlled under a certain level of guarantee $\alpha$. We can also compute the selective confidence interval for the DTW distance. 
The details are provided in Appendix \ref{appendix:selective_ci}.
To compute the selective $p$-value in \eq{eq:selective_p} as well as the selective confidence interval, we need to identify the conditional data space whose characterization will be introduced in the next section.

% ================================
\subsection{Conditional Data Space Characterization} 

We define the set of $(\bm X^\top ~ \bm Y^\top)^\top \in \RR^{n + m}$ that satisfies the conditions in \eq{eq:selective_p} as 
\begin{align} \label{eq:conditional_data_space}
\cD = 
\left \{ 
	{\bm X \choose \bm Y} \in \RR^{n + m}
	~ \Bigg | ~
	\begin{array}{l}
	\cA(\bm X, \bm Y) = \hat{M}^{\rm obs}, \\ 
	\cS(\bm X, \bm Y) = \hat{\bm s}^{\rm obs}, \\ 
	\cQ(\bm X, \bm Y) = \hat{\bm q}^{\rm obs}
	\end{array}  
\right \}. 
\end{align}
According to the third condition, 
%$\cQ(\bm X, \bm Y) = \hat{\bm q}^{\rm obs}$,
the data in $\cD$ is restricted to a line as stated in the following lemma.
\begin{lemma} \label{lemma:data_line}
The set $\cD$ in \eq{eq:conditional_data_space} can be rewritten using a scalar parameter $z \in \RR$ as follows:
\begin{align} \label{eq:conditional_data_space_line}
	\cD = \Big \{ (\bm X^\top ~ \bm Y^\top)^\top = \bm a + \bm b z \mid z \in \cZ \Big \},
\end{align}
where $\bm a = \cQ(\bm X^{\rm obs}, \bm Y^{\rm obs})$, $\bm b$ is defined in \eq{eq:q_and_b}, and
\begin{align} \label{eq:cZ}
	\cZ = \left \{ 
	z \in \RR ~
	~ \Big | ~
	\begin{array}{l}
	\cA(\bm a + \bm b z) = \hat{M}^{\rm obs}, \\ 
	\cS(\bm a + \bm b z) = \hat{\bm s}^{\rm obs}
	\end{array}
	\right \}.
\end{align}
Here, with a slight abuse of notation, 
$
\cA(\bm a + \bm b z) = \cA \left ((\bm X^\top ~ \bm Y^\top)^\top \right)
$
is equivalent to $\cA(\bm X, \bm Y)$.
This similarly applies to $\cS(\bm a + \bm b z)$.
\end{lemma}

\begin{proof}
The proof is deferred to Appendix \ref{appendix:proof_lemma_data_line}.
\end{proof}

Lemma \ref{lemma:data_line} indicates that we need NOT consider the $(n + m)$-dimensional space.
Instead, we need only consider the \emph{one-dimensional projected} space $\cZ$ in \eq{eq:cZ}.

\paragraph{Reformulation of selective $p$-value and identification of the truncation region $\cZ$.}
Let us consider a random variable and its observation:
\begin{align*}
	Z = \bm \eta_{\hat{M}, \hat{\bm s}}^\top {\bm X \choose \bm Y } \in \RR 
	~~ \text{and} ~~ 
	Z^{\rm obs} = \bm \eta_{\hat{M}, \hat{\bm s}}^\top {\bm X^{\rm obs} \choose \bm Y^{\rm obs} } \in \RR,
\end{align*}
%
%Let us consider a random variable $Z \in \RR$ and its observation $Z^{\rm obs} \in \RR$ that satisfies $(\bm X ~ \bm Y)^\top = \bm a + \bm b Z$ and $(\bm X^{\rm obs} ~ \bm Y^{\rm obs})^\top = \bm a + \bm b Z^{\rm obs}$.
%%
%The 
the selective $p$-value in (\ref{eq:selective_p}) can be rewritten as 
\begin{align} \label{eq:selective_p_parametrized}
	p_{\rm sel} 
%	& = \mathbb{P}_{\rm H_0} 
%	\left ( 
%	\bm \eta_{\hat{M}, \hat{\bm s}}^\top {\bm X \choose \bm Y} \geq 
%	\bm \eta_{\hat{M}, \hat{\bm s}}^\top {\bm X^{\rm obs} \choose \bm Y^{\rm obs}}
%	~ \Big | ~
%	{\bm X \choose \bm Y} \in \cD
%	\right) \nonumber \\ 
%	& 
	= \mathbb{P}_{\rm H_0} \left ( Z \geq Z^{\rm obs}
	\mid 
	Z \in \cZ
	\right).
\end{align}
Because $Z \sim \NN \big (0, \bm \eta_{\hat{M}, \hat{\bm s}}^\top \Sigma \bm \eta_{\hat{M}, \hat{\bm s}} \big )$ under the null hypothesis, $Z \mid Z \in \cZ$ follows a \emph{truncated} normal distribution. 
Once 
%the truncation region 
$\cZ$ is identified, computing the $p_{\rm sel}$ in (\ref{eq:selective_p_parametrized}) is straightforward.
Therefore, the remaining task is to identify the truncation region $\cZ$ in \eq{eq:cZ}, which can be decomposed into two sets as 
$\cZ = \cZ_1 \cap \cZ_2$,  where 
\begin{align}
	\cZ_1 & = \{ z \in \RR \mid  \cA(\bm a + \bm b z) = \hat{M}^{\rm obs}\} \label{eq:cZ_1}\\  
	~ \text{ and } ~
	\cZ_2 & = \{ z \in \RR \mid \cS(\bm a + \bm b z) = \hat{\bm s}^{\rm {obs}} \}. \label{eq:cZ_2}
\end{align}
The constructions of $\cZ_1$ and $\cZ_2$ are presented in \S \ref{sec:computational_method}.

%\begin{figure}[!t]
%\centering
%\includegraphics[width=\linewidth]{connection_si_lasso_dtw.pdf}  
%\caption{
%The connection between the proposed method and the seminal conditional SI study \cite{lee2016exact}.
%}
%\label{fig:connection_si_lasso_dtw}
%%\vspace{-0.5cm}
%\end{figure}

\paragraph{Connections to conditional SI literature.} 
The proposed method draws from the ideas of the conditional SI literature and the connections are as follows:

%$\bullet$ \red{Conditioning on the optimal alignment $\hat{M}^{\rm obs}$ and the signs $\hat{\bm s}^{\rm obs}$ in \eq{eq:conditional_inference} corresponds to conditioning on the selected features and their signs in \cite{lee2016exact}.
%%
%They also corresponds to conditioning on the transportation plan and their signs in \cite{duy2021exact} (see Fig. \ref{fig:connection_si_lasso_wasserstein_dtw}).}

$\bullet$ Conditioning on $\hat{M}^{\rm obs}$ and the signs $\hat{\bm s}^{\rm obs}$ in \eq{eq:conditional_inference} corresponds to conditioning on the selected features and their signs in \cite{lee2016exact} as well as conditioning on the transportation plan and their signs in \cite{duy2021exact} (see Fig. \ref{fig:connection_si_lasso_wasserstein_dtw}).

$\bullet$ The $\cQ(\bm X, \bm Y)$ in \eq{eq:q_and_b} corresponds to the component $\bm z$ in \cite{lee2016exact} (see Sec. 5, Eq. 5.2 and Theorem 5.2).
Additional conditioning on $\cQ(\bm X, \bm Y)$ is a standard approach in the conditional SI literature.

$\bullet$ The fact of restricting the data to the line in Lemma \ref{lemma:data_line} has been already implicitly exploited in \cite{lee2016exact}, but explicitly discussed in Sec. 6 of \cite{liu2018more}.

\begin{figure}[!t]
\centering
\includegraphics[width=\linewidth]{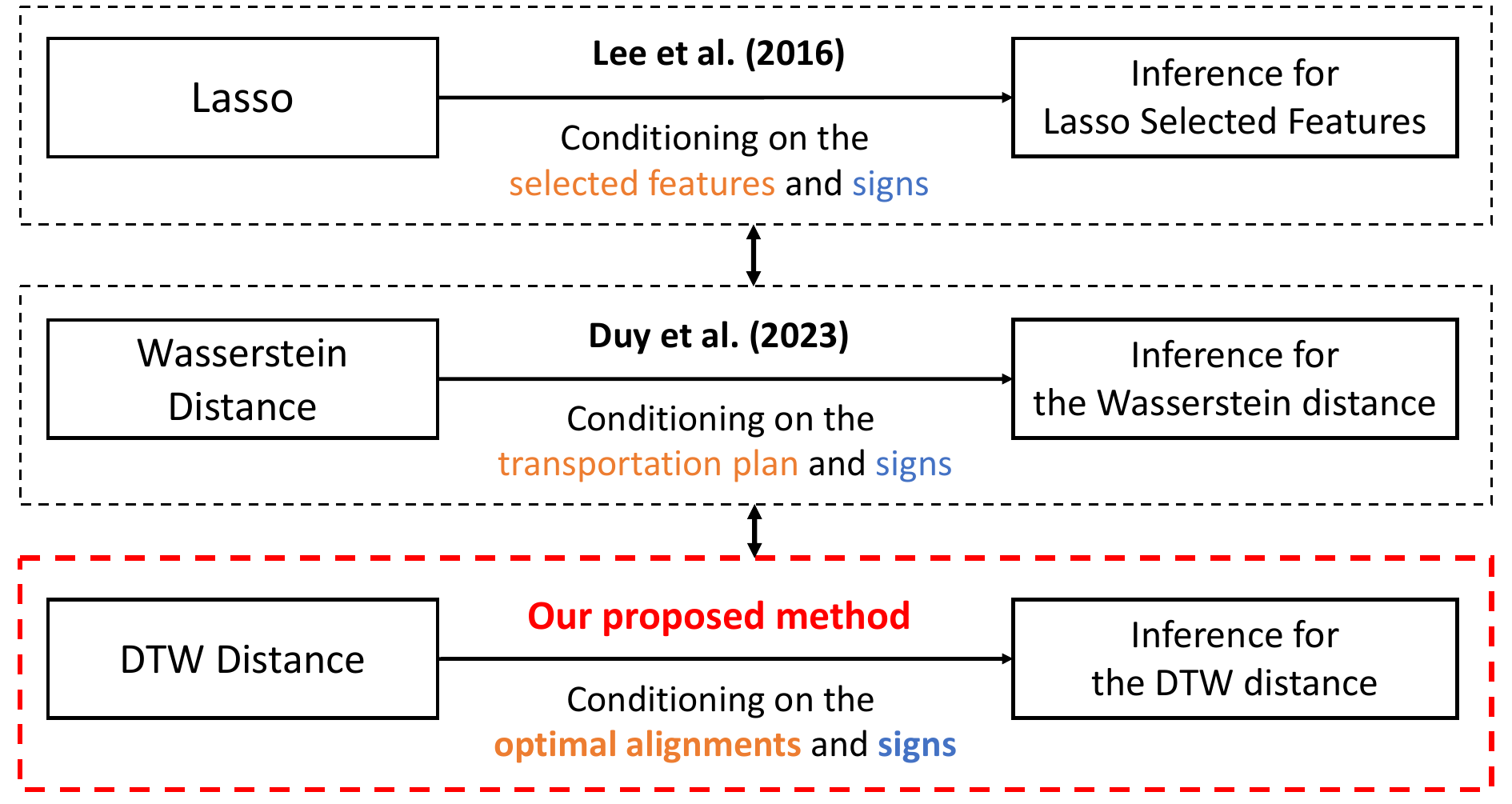}  
\caption{
The connection between the proposed method, the seminal conditional SI study \cite{lee2016exact}, and \cite{duy2021exact}.
}
\label{fig:connection_si_lasso_wasserstein_dtw}
%\vspace{-0.5cm}
\end{figure}

%% file: sec4.tex
%\section{Computational Method for Computing $\cZ$} 
\section{Computation of $\cZ$} \label{sec:computational_method}

In this section, we present our second contribution of  introducing a novel computational method, called \emph{parametric DTW}, to compute the truncation region $\cZ$.
The basic idea is illustrated in Fig. \ref{fig:sec_4_Z}.

\begin{figure}[!t]
\centering
\includegraphics[width=\linewidth]{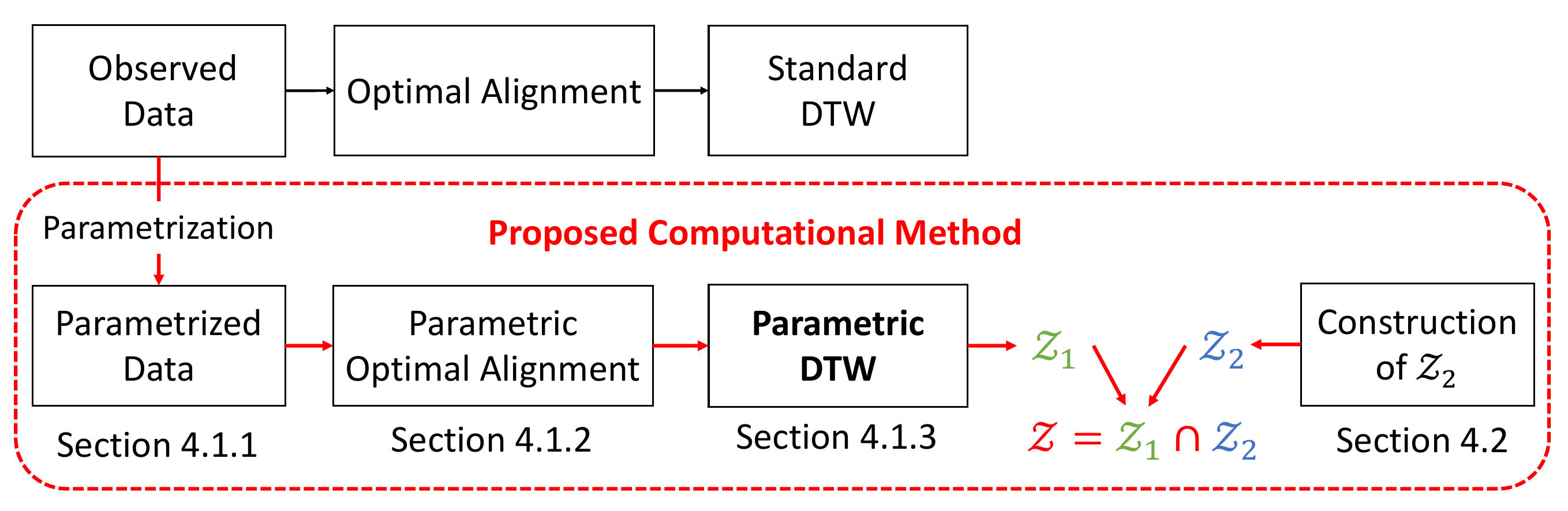}  
\caption{
Schematic illustration of computing $\cZ$.
}
\label{fig:sec_4_Z}
\end{figure}

\subsection{Construction of $\cZ_1$ in \eq{eq:cZ_1}}

\subsubsection{Parametrization of time-series data}

\paragraph{Important notations.}

Before discussing the construction of $\cZ_1$, we introduce some notations.
As mentioned in Lemma \ref{lemma:data_line}, we focus on a set of data $(\bm X^\top ~ \bm Y^\top)^\top = \bm a + \bm b z \in \RR^{n + m}$. % parametrized by a scalar parameter $z$. 
We denote 
\begin{align} \label{eq:Z_z_Y_z}
	\hspace{-5pt}
	\bm X (z) = \bm a^{(1)} + \bm b^{(1)} z
	\quad 
	\text{and}
	\quad
	\bm Y (z) = \bm a^{(2)} + \bm b^{(2)} z,
\end{align}
where 
$ \bm a^{(1)} = \bm a_{1:n} \sqsubseteq \bm a$ is a sub-sequence of $\bm a \in \RR^{n + m}$ from positions 1 to $n$, 
\begin{align*} 
\bm b^{(1)} = \bm b_{1:n}, 
\quad \bm a^{(2)} = \bm a_{n+1:n + m}, 
\quad \bm b^{(2)} = \bm b_{n+1:n + m}.
\end{align*} 
Then, the parametrized cost matrix 
%$C \Big ( \bm X(z), \bm Y(z) \Big )$ 
is defined as 
\begin{footnotesize}
\begin{align*} %\label{eq:cost_matrix_parametrized}
	C\Big (\bm X (z), \bm Y (z) \Big ) 
	& = \left[
	\Big (
	\left (
	 a^{(1)}_i + 
	 b^{(1)}_i z 
	\right)
	-
	\left(
	 a^{(2)}_j +
	 b^{(2)}_j z
	\right )
	\Big)^2
	\right]_{ij}. %\in \RR^{n \times m}.
\end{align*}
\end{footnotesize}
\hspace{-2pt}Given $M \in \cM_{n, m}$, $\bm X(z) \in \RR^n$ and $\bm Y(z) \in \RR^m$, the loss function for the optimal alignment problem is a \emph{quadratic function (QF)}  w.r.t. $z$ and it is written as 
\begin{align} \label{eq:cost_parametrized}
	L_{n, m} \big (M, z \big ) 
	&= 
	\Big \langle
		M, C\big (\bm X (z), \bm Y (z) \big )
	\Big \rangle \nonumber \\ 
%	& = 
%	\sum_{i \in [n], j \in [m]} 
%	M_{ij} C_{ij}\big (\bm X (z), \bm Y (z) \big ) \nonumber \\ 
	& = \omega_0 + \omega_1 z + \omega_2  z^2,
\end{align}
where $\omega_0, \omega_1, \omega_2 \in \RR$ and they are defined as 
\begin{footnotesize}
\begin{align*}
	\hspace{-1pt}
	\omega_0 &= \sum_{i, j} M_{ij} \left (a^{(1)}_i - a^{(2)}_j\right)^2,  ~
	\omega_2 = \sum_{i, j} M_{ij} \left (b^{(1)}_i - b^{(2)}_j\right)^2, \\ 
	\hspace{-1pt} 
	\omega_1 & = 2 \sum_{i, j} M_{ij} \left (a^{(1)}_i - a^{(2)}_j\right) \left (b^{(1)}_i - b^{(2)}_j\right).
\end{align*}
\end{footnotesize}
\hspace{-5pt}The optimal alignment in \eq{eq:optimal_alignment} and the DTW distance on parametrized data  $\big ( \bm X(z), \bm Y(z) \big )$ is defined as
\begin{align} 
	\hat{M}_{n, m} (z) & = \argmin \limits_{M \in \cM_{n, m}} L_{n, m} \big (M, z \big ),  \label{eq:optimal_alignment_parametrized} \\ 
	\hat{L}_{n, m}(z) & = \min \limits_{M \in \cM_{n, m}} L_{n, m} \big (M, z \big ).  \label{eq:optimal_cost_parametrized}
\end{align}

\textbf{Construction of $\cZ_1$.} The $\cZ_1$ %in \eq{eq:cZ_1} 
can be re-written as 
\begin{align*}
	\cZ_1 & = \left \{ z \in \RR \mid  \cA\big (\bm X(z), \bm Y(z) \big ) = \hat{M}^{\rm obs} \right \} \\ 
	& = \left \{ z \in \RR \mid  \hat{M}_{n, m}(z) = \hat{M}^{\rm obs}\right \}.
\end{align*}
To compute $\cZ_1$, we have  two computational challenges:

$\bullet$ \emph{Challenge 1}: we need to compute the \emph{entire path} of the optimal alignment matrix $\hat{M}_{n, m}(z)$ for all $z \in \RR$.
However, it seems intractable because we have to solve \eq{eq:optimal_alignment_parametrized} for \emph{infinitely} many values of $z \in \RR$ to obtain $\hat{M}_{n, m}(z)$ and check if it is the same as $\hat{M}^{\rm obs}$ or not.
	
$\bullet$ \emph{Challenge 2}: we have to solve \eq{eq:optimal_alignment_parametrized} on a huge set of all alignment matrices $\cM_{n, m}$ that grows exponentially.
% whose size is exponentially increasing with $n$ and $m$.

%
In \S \ref{subsubsec:para_OA}, we introduce an efficient approach to resolve the first challenge.
We show that the set $\cZ_1$ can be computed with \emph{a finite number of operations}.
Finally, in \S \ref{subsubsec:para_DTW}, we propose a method to address the second challenge based on the concept of dynamic programming in the standard DTW.

% ================================
\subsubsection{Parametric Optimal Alignment} \label{subsubsec:para_OA}

\begin{algorithm}[!t]
\renewcommand{\algorithmicrequire}{\textbf{Input:}}
\renewcommand{\algorithmicensure}{\textbf{Output:}}
\begin{scriptsize}
 \begin{algorithmic}[1]
  \REQUIRE $n, m, \cM_{n, m} \quad \quad \quad \quad $  
  \vspace{4pt}
  \STATE $t \lA 1$, $z_1 \lA -\infty$
  \vspace{4pt}
  \STATE $\hat{M}_t \lA \hat{M}_{n, m} (z_t) = \argmin \limits_{M \in \cM_{n, m}} L \big (M, z_t \big )$
  \WHILE { $z_t < +\infty$}
  \vspace{4pt}
  \STATE Find the next breakpoint $z_{t+1} > z_t$ and the next optimal alignment matrix $\hat{M}_{t+1}$ s.t. 
  \vspace{4pt}
  \begin{center}
   $L_{n, m}(\hat{M}_t, z_{t+1}) = L_{n, m}(\hat{M}_{t + 1}, z_{t+1}).$
  \end{center}
  \vspace{4pt}
  \STATE $t \lA t+1$
  \vspace{4pt}
  \ENDWHILE
  \vspace{4pt}
  \STATE $\cT \lA t$
  \vspace{4pt}
  \ENSURE $\big \{\hat{M}_t \big \}_{t=1}^{\cT - 1}$, $\big \{z_t \big \}_{t=1}^{\cT}$
 \end{algorithmic}
\end{scriptsize}
\caption{{\tt paraOA}$(n, m, \cM_{n, m})$}
\label{alg:paraOptAlign}
\end{algorithm}

\begin{figure}[!t]
\centering
\includegraphics[width=.75\linewidth]{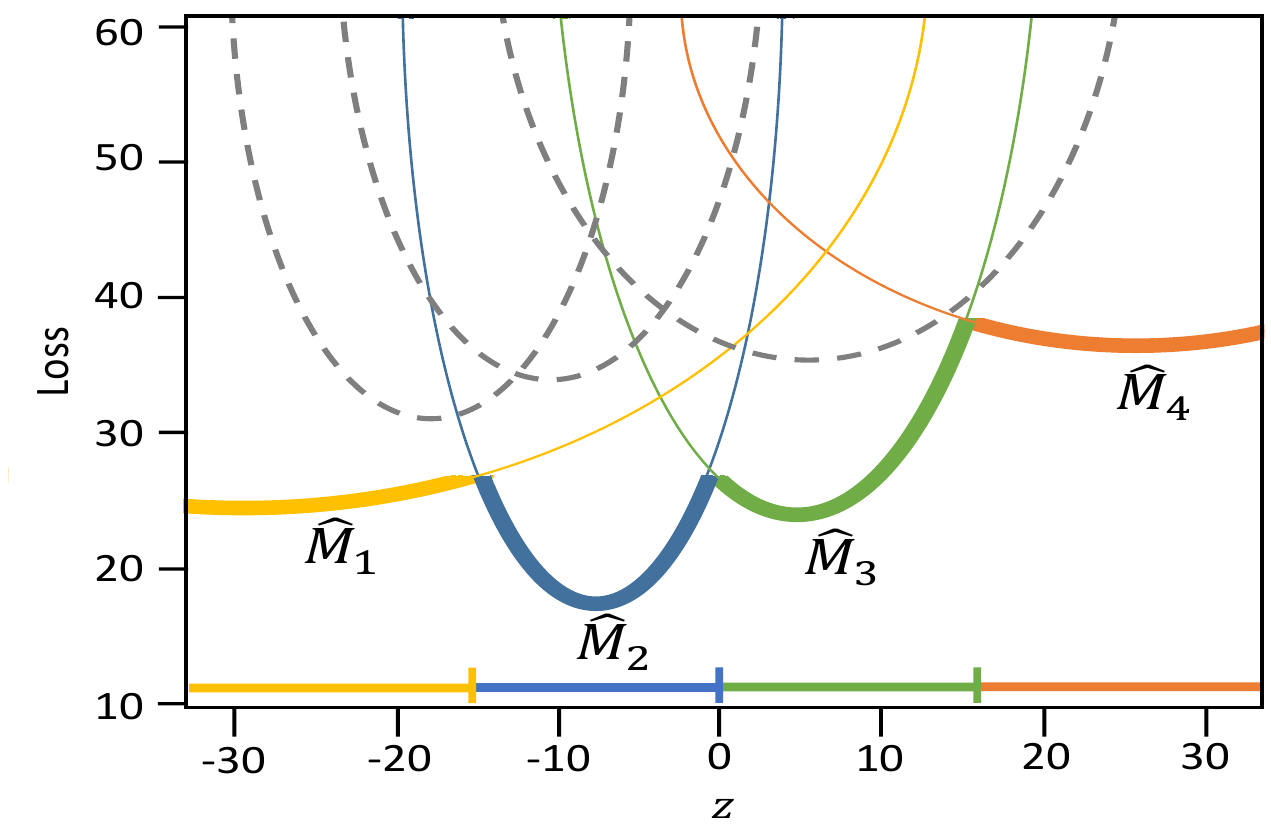}
\caption{
\footnotesize A set of QFs each of which corresponds to an alignment matrix $M \in \cM_{n, m}$. 
The dotted grey QFs correspond to alignment matrices that are NOT optimal for any $z \in \RR$. A set $\{\hat{M}_1, \hat{M}_2, \hat{M}_3, \hat{M}_4\}$ contains matrices that are \emph{optimal} for some $z \in \RR$.
Our goal is to introduce an approach to efficiently identify this set of optimal alignment matrices and the lower envelope.}
\label{fig:piecewise}
\end{figure}

Algorithm~\ref{alg:paraOptAlign} shows the proposed parametric optimal alignment method.
Here, %we exploit the fact that, 
for each alignment matrix $M \in \cM_{n, m}$,
the loss function $L_{n, m}(M, z)$ is written as a QF of $z$ as in \eq{eq:cost_parametrized}.
Since the number of matrices $M$ in
$\cM_{n, m}$ is finite, the optimal alignment problem \eq{eq:optimal_cost_parametrized} can be characterized by a finite number of these QFs.

Figure \ref{fig:piecewise} illustrates the set of QFs
each of which corresponds to an alignment matrix $M \in \cM_{n, m}$.
Since the minimum loss for each $z \in \RR$ is the point-wise minimum of these QFs, the $\hat{L}_{n, m}(z)$ in \eq{eq:optimal_cost_parametrized} 
is the lower envelope of the set of QFs 
that is a \emph{piecewise QF} of $z$.
Parametric optimal alignment is interpreted as the problem of identifying this piecewise QF. 

In Algorithm \ref{alg:paraOptAlign}, multiple 
\emph{breakpoints}
$
 z_1 < z_2 < \ldots < z_{\cT}
$
are computed one by one.
Each breakpoint $z_t, t \in [\cT],$ indicates a point at which
the optimal alignment matrix changes, where $\cT$ is the number of breakpoints. % determined by the algorithm.
By finding all these breakpoints %$\{z_t\}_{t=1}^{\cT}$
and the optimal alignment matrices, 
%$\{\hat{M}_t\}_{t=1}^{\cT - 1}$,
the piecewise QF $\hat{L}_{n, m}(z)$ as in Fig. \ref{fig:piecewise} (the curves in yellow, blue, green and orange) can be identified. 
Finally, the entire path of optimal alignment matrices for $z \in \RR$ 
is given by
\begin{align*}
	 \hat{M}_{n, m}(z)	=  \hat{M}_t, ~ t \in [\cT - 1],
	 \text{ if }
	 z \in [z_t, z_{t + 1}].
\end{align*}
%
%Due to space limitations, more detailed explanations of Algorithm \ref{alg:paraOptAlign} are deferred to Appendix \ref{appendix:details_algorithm_paraOptAlign}.
%
More details are deferred to Appendix \ref{appendix:details_algorithm_paraOptAlign}.

%\begin{figure}[!t]
%\centering
%\includegraphics[width=.7\linewidth]{piecewise_1.pdf}
%\caption{A set of quadratic functions (QFs) each of which corresponds to an alignment matrix $M \in \cM_{n, m}$. 
%%
%The dotted grey QFs correspond to alignment matrices that are NOT optimal for any $z \in \RR$. A set $\{\hat{M}_1, \hat{M}_2, \hat{M}_3, \hat{M}_4\}$ contains alignment matrices that are \emph{optimal} for some $z \in \RR$.
%%
%Our goal is to introduce an approach to efficiently identify this set of optimal alignment matrices and the lower envelope.}
%\label{fig:piecewise}
%\end{figure}

% ================================
\subsubsection{Parametric DTW} \label{subsubsec:para_DTW}

Unfortunately, Algorithm \ref{alg:paraOptAlign}
%with the inputs $n$, $m$ and $\cM_{n, m}$ 
is impractical because the cardinality of $\cM_{n, m}$ is exponentially increasing with
$n$ and $m$. 
To address the issue, we utilize the concept of the standard DTW and apply it to the parametric case, which we call \emph{parametric DTW}.
The idea is to exclude the alignment matrices $M \in \cM_{n, m}$ which can never be optimal at any $z \in \RR$.
Instead of considering a huge set $\cM_{n, m}$, we only construct a smaller set $\tilde{\cM}_{n, m}$.
%
%Before introducing the details of parametric DTW, we briefly review the standard DTW.
We briefly review the standard DTW as follows.

\paragraph{Standard DTW (for a single value of $z$).}
In the standard DTW with $n$ and $m$, we use $n \times m$ table 
whose
$(i, j)^{\rm th}$ element contains $\hat{M}_{i, j}(z)$ that is 
the optimal alignment matrix for the sub-sequences
$\bm X(z)_{1:i}$ and $\bm Y(z)_{1:j}$. 
The optimal alignment matrix $\hat{M}_{i, j}(z)$ for each sub-problem with $i$ and $j$ is used for efficiently computing the optimal alignment matrix $\hat{M}_{n, m}(z)$ for the original problem with $n$ and $m$ by using \emph{Bellman equation} (see Appendix \ref{appendix:review_standard_DTW} for the details). 
%
%More detailed review of the standard DTW is provided in Appendix \ref{appendix:review_standard_DTW}.

\paragraph{Parametric DTW (for all $z \in \RR$).} We construct an
$n \times m$ table whose $(i, j)^{\rm th}$ element contains
\begin{align*}
 \hat{\cM}_{i, j}
 =
 \left\{
M \in \cM_{i, j}
 \mid
 {\exists}z \in \RR
 \text{ s.t. }
 \hat{L}_{i, j}(z) = L_{i, j}(M, z) 
\right\}
\end{align*}
which is a \emph{set of optimal alignment matrices} that are optimal for some $z$. 
%
%Because the optimal alignment matrices $\hat{M}(z)$ is fixed between two consecutive breakpoints $z_t$ and $z_{t+1}$ for $t \in [\cT-1]$, it is given as, $ \hat{\cM}_{i, j} = \big \{\hat{M}_t \big \}_{t=1}^{\cT - 1}$ that is returned by Algorithm \ref{alg:paraOptAlign}.
%
For example, $\hat{\cM}_{i, j}$ is a set $\big \{\hat{M}_1, \hat{M}_2, \hat{M}_3, \hat{M}_4 \big \}$ in Fig. \ref{fig:piecewise}.
To identify $\hat{\cM}_{i, j}$, we construct a set $\tilde{\cM}_{i, j} \supseteq \hat{\cM}_{i, j}$, which is a set of alignment matrices having potential to be optimal at some $z$.
The construction of $\hat{\cM}_{i, j}$ is described as follows.
%in the following lemma.

\begin{lemma} \label{lemma:bellman_parametric}
For $i \in [n]$ and $j \in [m]$, the set of optimal alignment matrices $\hat{\cM}_{i, j}$ is defined as 
\begin{align} \label{eq:bellman_parametric_opt_matrices}
	\hat{\cM}_{i, j} = \argmin \limits_{M \in \tilde{\cM}_{i, j}} L_{i, j} \big (M, z \big ),
\end{align}
where $\tilde{\cM}_{i, j}$ is a set of alignment matrices having potential to be optimal and it is constructed as
\begin{footnotesize}
\begin{align*} %\label{eq:potential_optimal_set}
	\tilde{\cM}_{i, j} = 
	\left \{ 
	\begin{array}{l}
		{\rm vstack} \Big (\hat{M}, ~ (0, ...,0, 1) \Big ), ~ \forall \hat{M} \in \hat{\cM}_{i - 1, j},\\ 
		{\rm hstack} \Big (\hat{M}, ~ (0, ...,0, 1)^\top \Big ), ~ \forall \hat{M} \in \hat{\cM}_{i, j - 1}, \\ 
		\begin{pmatrix}
			\hat{M} & ~ 0 \\ 
			0 & ~ 1 \\ 
		\end{pmatrix},
		~ \forall \hat{M} \in \hat{\cM}_{i - 1, j - 1}
	\end{array}
	\right \}.
\end{align*}
\end{footnotesize}
\end{lemma}

%\begin{lemma} \label{lemma:bellman_parametric}
%For $i \in [n]$ and $j \in [m]$, the set of optimal alignment matrices $\hat{\cM}_{i, j}$ is defined as 
%%
%\begin{align} \label{eq:bellman_parametric_opt_matrices}
%	\hat{\cM}_{i, j} = \argmin \limits_{M \in \tilde{\cM}_{i, j}} ~ L_{i, j} \big (M, z \big ),
%\end{align}
%%
%where $\tilde{\cM}_{i, j}$ is a set of alignment matrices having potential to be optimal and it is constructed as
%%
%\begin{footnotesize}
%\begin{align*} %\label{eq:potential_optimal_set}
%	\tilde{\cM}_{i, j} = 
%	\left \{ 
%	\begin{array}{l}
%		{\rm vstack} \Big (\hat{M}, ~ (0, ...,0, 1) \Big ), ~ \forall \hat{M} \in \hat{\cM}_{i - 1, j}, 
%		\begin{pmatrix}
%			\hat{M} & ~ 0 \\ 
%			0 & ~ 1 \\ 
%		\end{pmatrix},
%		~ \forall \hat{M} \in \hat{\cM}_{i - 1, j - 1},
%		\\ 
%		{\rm hstack} \Big (\hat{M}, ~ (0, ...,0, 1)^\top \Big ), ~ \forall \hat{M} \in \hat{\cM}_{i, j - 1}
%	\end{array}
%	\right \}.
%\end{align*}
%\end{footnotesize}
%%\begin{footnotesize}
%%\begin{align*} %\label{eq:potential_optimal_set}
%%	\tilde{\cM}_{i, j} = 
%%	\left \{ 
%%	\begin{array}{l}
%%		{\rm vstack} \Big (\hat{M}, ~ (0, ...,0, 1) \Big ), ~ \forall \hat{M} \in \hat{\cM}_{i - 1, j},\\ 
%%		{\rm hstack} \Big (\hat{M}, ~ (0, ...,0, 1)^\top \Big ), ~ \forall \hat{M} \in \hat{\cM}_{i, j - 1}, \\ 
%%		\begin{pmatrix}
%%			\hat{M} & ~ 0 \\ 
%%			0 & ~ 1 \\ 
%%		\end{pmatrix},
%%		~ \forall \hat{M} \in \hat{\cM}_{i - 1, j - 1}
%%	\end{array}
%%	\right \}.
%%\end{align*}
%%\end{footnotesize}
%\end{lemma}
%

%\begin{proof}
%The proof is deferred to Appendix \ref{appendix:proof_lemma_bellman_parametric}.
%\end{proof}

The proof of Lemma \ref{lemma:bellman_parametric} is deferred to Appendix \ref{appendix:proof_lemma_bellman_parametric}. From Lemma \ref{lemma:bellman_parametric}, we efficiently construct $\tilde{\cM}_{i, j}$.
Then, $\tilde{\cM}_{i, j}$ is used to compute $\hat{\cM}_{i, j}$ by {\tt paraOA}$(i, j, \tilde{\cM}_{i, j})$ in Algorithm \ref{alg:paraOptAlign}.
%
%Here, we note that Algorithm \ref{alg:paraOptAlignOnSubProblem} is similar to Algorithm \ref{alg:paraOptAlign} which is presented in Sec 4.1.1.
%%
%The main difference is that we now solving parametric optimal alignment on smaller problem with $i, j$ and $\tilde{\cM}_{i, j}$.
%
%By repeating the recursive procedure and storing 
%$\hat{\cM}_{i, j}$
%in the $(i, j)^{\rm th}$ element of the table 
%from smaller $i$ and $j$ to larger $i$ and $j$, 
%we can end up with
%$\tilde{\cM}_{n, m} \supseteq \hat{\cM}_{n, m}$.
By repeating the recursive procedure
from smaller $i$ and $j$ to larger $i$ and $j$, 
we can end up with
$\tilde{\cM}_{n, m} \supseteq \hat{\cM}_{n, m}$.
%
%By using parametric DTW,
The set
$\tilde{\cM}_{n, m}$
can be much smaller than 
%the size of all possible alignment matrices
%that of 
$\cM_{n, m}$,
which makes
the cost of
${\tt paraOA}(n, k, \tilde{\cM}_{n, m})$ substantially decreased
compared to
${\tt paraOA}(n, k, \cM_{n, m})$.
% in Algorithm \ref{alg:paraOptAlign} of Sec 4.1.1.
%
The parametric DTW is presented in Algorithm \ref{alg:paraDTW} whose output is used to identify 
\[
	\cZ_1 = \mathop{\cup}_{\hat{M}_{n, m}(z) \in \hat{\cM}_{n, m}} \left \{z: \hat{M}_{n, m}(z) = \hat{M}^{\rm obs} \right \}.
\]
%
%Finally, based on the output $\hat{\cM}_{n, m}$, the $\cZ_1$ is identified by
%%
%$
%	\cZ_1 = \mathop{\cup}_{\hat{M}_{n, m}(z) \in \hat{\cM}_{n, m}} \left \{z: \hat{M}_{n, m}(z) = \hat{M}^{\rm obs} \right \}.
%$

\paragraph{Complexity.}
The complexity of the parametric DTW in Algorithm \ref{alg:paraDTW} is $\cO(n \times m \times \delta)$, 
where $\delta$ is the number of breakpoints in  Algorithm \ref{alg:paraOptAlign}.
In the worst-case, the value of $\delta$ still grows exponentially. 
This is a common issue in other parametric programming applications such as Lasso regularization path. 
However, fortunately, it has been well-recognized that this worst case rarely happens, and the value of $\delta$ is almost linearly increasing w.r.t the problem size in practice (e.g., \cite{le2021parametric}).
This phenomenon is well-known in the parametric programming literature~\cite{hastie2004entire,park2007l1,mairal2012complexity}.

%=========== ParaDP ===========
\begin{algorithm}[!t]
\renewcommand{\algorithmicrequire}{\textbf{Input:}}
\renewcommand{\algorithmicensure}{\textbf{Output:}}
\begin{scriptsize}
 \begin{algorithmic}[1]
  \REQUIRE $\bm X(z)$ and $\bm Y(z)$
  \vspace{1pt}
  \FOR{$i=1$ to $n$}
  \vspace{1pt}
  \FOR{$j=1$ to $m$}
  \vspace{2pt}
  \STATE $\tilde{\cM}_{i, j}$ $\lA$ Lemma \ref{lemma:bellman_parametric}
  \vspace{3pt}
  \STATE \hspace{-2pt}$ \{\hat{M}_t  \}_{t=1}^{\cT - 1}$, $ \{z_t  \}_{t=1}^{\cT}$ $\lA$ {\tt paraOA}($i, j, \tilde{\cM}_{i, j}$) $~$ // Algorithm \ref{alg:paraOptAlign}
  \vspace{3pt}
  \STATE $\hat{\cM}_{i, j} \lA \{\hat{M}_t  \}_{t=1}^{\cT - 1}$ 
  \vspace{3pt}
  \ENDFOR
  \vspace{2pt}
  \ENDFOR
  \vspace{3pt}
  \ENSURE $\hat{\cM}_{n, m}$
 \end{algorithmic}
\end{scriptsize}
\caption{{\tt paraDTW}($\bm X(z), \bm Y(z)$)}
\label{alg:paraDTW}
\end{algorithm}

\begin{algorithm}[!t]
\renewcommand{\algorithmicrequire}{\textbf{Input:}}
\renewcommand{\algorithmicensure}{\textbf{Output:}}
\begin{scriptsize}
 \begin{algorithmic}[1]
  \REQUIRE $\bm X^{\rm obs}$ and $\bm Y^{\rm obs}$
  \vspace{2pt}
  \STATE $\hat{M}^{\rm obs} \lA \cA(\bm X^{\rm obs}, \bm Y^{\rm obs})$
  \vspace{2pt}
  \STATE $\bm X(z)$ and $\bm Y(z)$ $\lA$ Eq. \eq{eq:Z_z_Y_z} 
  \vspace{2pt}
  \STATE $\hat{\cM}_{n, m}$ $\lA$ {\tt paraDTW}($\bm X(z)$, $\bm Y(z)$)  $\quad$ // Algorithm \ref{alg:paraDTW}
  \vspace{2pt}
  \STATE $\cZ_1 \lA \mathop{\cup}_{\hat{M}_{n, m}(z) \in \hat{\cM}_{n, m}} \{z: \hat{M}_{n, m}(z) = \hat{M}^{\rm obs}\}$
  \vspace{2pt}
  \STATE $\cZ_2 \lA$ Lemma \ref{lemma:cZ_2}%Eq. \eq{eq:identification_cZ_2}
  \vspace{2pt}
  \STATE $\cZ = \cZ_1 \cap \cZ_2$
  \vspace{2pt}
  \STATE $p_{\rm selective}$ $\lA$ Eq. \eq{eq:selective_p_parametrized}
  \vspace{2pt}
  \ENSURE $p_{\rm selective}$
 \end{algorithmic}
\end{scriptsize}
\caption{Proposed SI Method (SI-DTW)}
\label{alg:si_dtw}
\end{algorithm}

\subsection{Construction of $\cZ_2$ in \eq{eq:cZ_2}}
We present the construction of $\cZ_2$ as follows.

\begin{lemma} \label{lemma:cZ_2}
The set $\cZ_2$ in \eq{eq:cZ_2} is an interval:
\begin{align} \label{eq:identification_cZ_2}
	\hspace{-5pt} \cZ_2 = \left \{ z ~ \Big | ~ 
	\max \limits_{j: \nu_j^{(2)} > 0} \frac{ - \nu_j^{(1)}}{\nu_j^{(2)}}
	\leq z \leq
	\min \limits_{j: \nu_j^{(2)} < 0} \frac{ - \nu_j^{(1)}}{\nu_j^{(2)}}
	\right \},
\end{align}
where $\bm \nu^{(1)} = \hat{\bm s}^{\rm obs}  \circ \hat{M}_{\rm vec} \circ \Omega \bm a$,
%and  
$\bm \nu^{(2)} = \hat{\bm s}^{\rm obs}  \circ \hat{M}_{\rm vec} \circ \Omega \bm b$.
\end{lemma}

%\begin{proof}
%The proof is deferred to Appendix \ref{appendix:proof_lemma_cZ_2}.
%\end{proof}

The proof of Lemma \ref{lemma:cZ_2} is deferred to Appendix \ref{appendix:proof_lemma_cZ_2}. After computing $\cZ_2$, we obtain $\cZ = \cZ_1 \cap \cZ_2$ and compute the selective $p$-value in \eq{eq:selective_p_parametrized}.
% for conducting the inference.
%%
%The entire proposed SI-DTW method for computing selective $p$-values is summarized in Appendix \ref{appendix:entire_algorithm}.
%
The entire proposed %SI-DTW 
method 
%for computing selective $p$-values 
is summarized in Algorithm~\ref{alg:si_dtw}.

%% file: sec5.tex
\section{Experiment} \label{sec:experiment}

\begin{figure}[!t]
\begin{subfigure}{.495\linewidth}
  \centering
  \includegraphics[width=\linewidth]{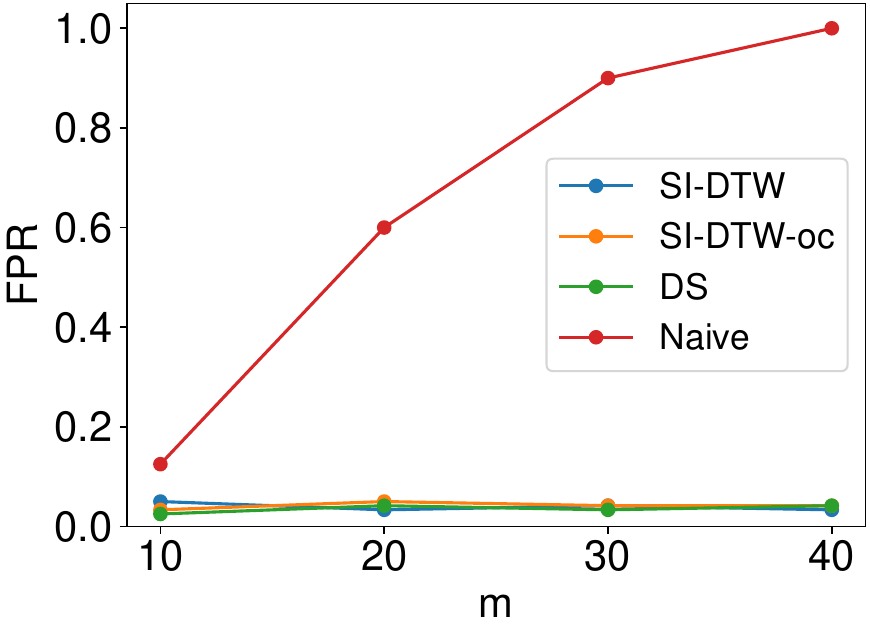}  
  \caption{Independence}
\end{subfigure}
\begin{subfigure}{.495\linewidth}
  \centering
  \includegraphics[width=\linewidth]{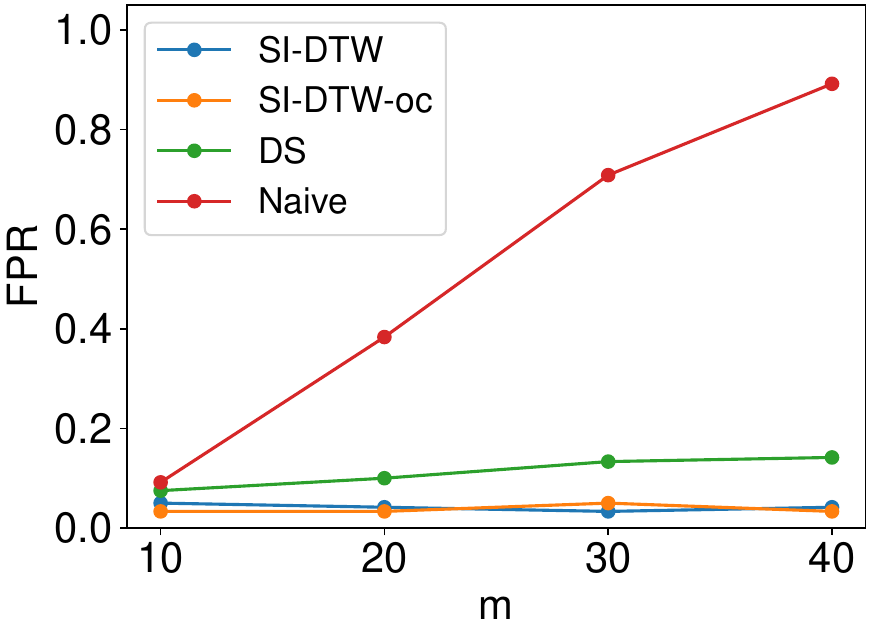} 
  \caption{Correlation}
\end{subfigure}
%
%\begin{subfigure}{.245\linewidth}
%  \centering
%  \includegraphics[width=\linewidth]{coverage_independence}  
%  \caption{CI Coverage (case 1)}
%\end{subfigure}
%\begin{subfigure}{.245\linewidth}
%  \centering
%  \includegraphics[width=\linewidth]{coverage_correlation} 
%  \caption{CI Coverage (case 2)}
%\end{subfigure}
\caption{FPR Comparison} 
\label{fig:fpr_ci_coverage}
\end{figure}

\begin{figure}[!t]
\begin{subfigure}{.495\linewidth}
  \centering
  \includegraphics[width=\linewidth]{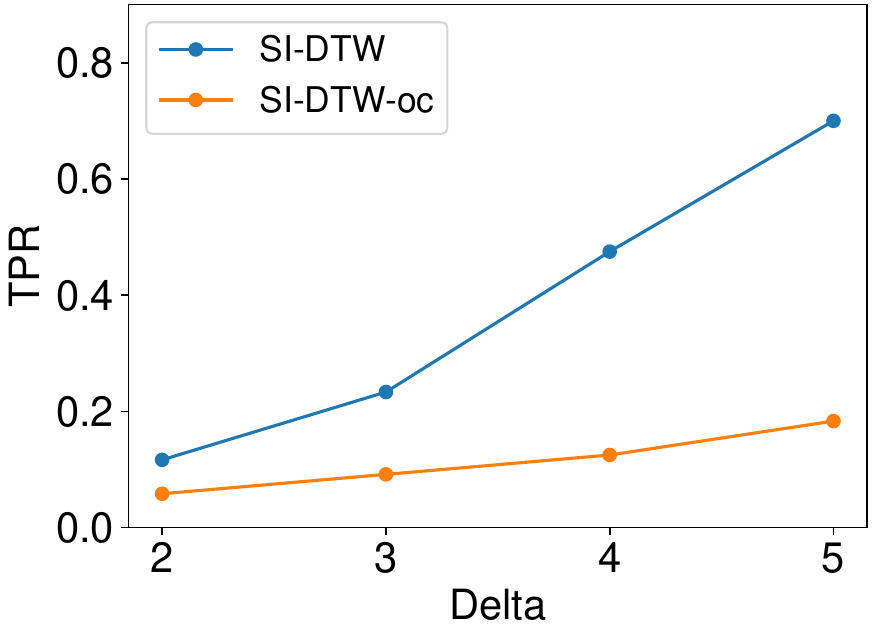}  
  \caption{Independence}
\end{subfigure}
\begin{subfigure}{.495\linewidth}
  \centering
  \includegraphics[width=\linewidth]{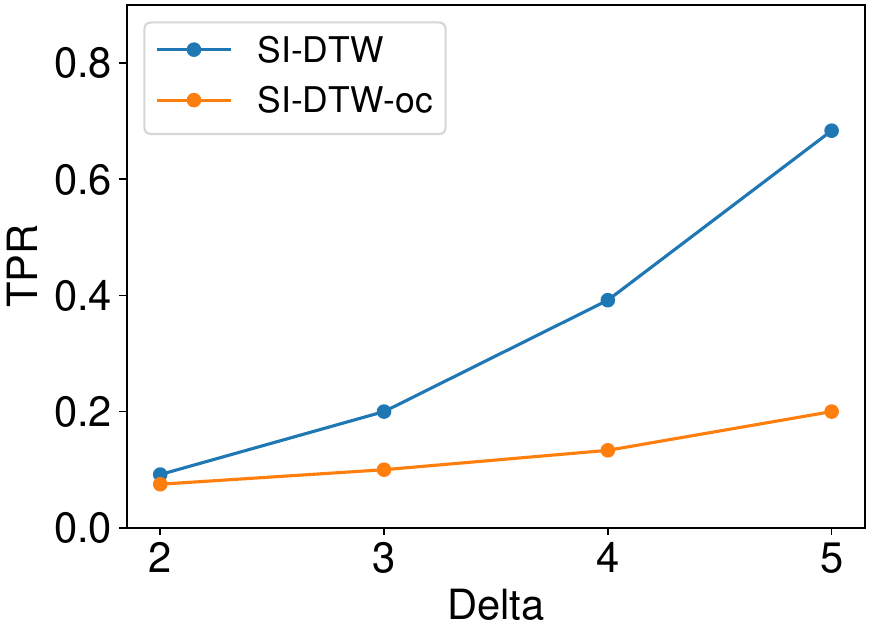} 
  \caption{Correlation}
\end{subfigure}
\caption{TPR comparison} 
\label{fig:tpr}
%\vspace{-20pt}
\end{figure}

%In this section, we present synthetic data experiments (\S \ref{subsec:synthetic_experiments}) to confirm the validity and the power of the proposed method and real data experiments (\S \ref{subsec:real_data_experiments}) to demonstrate the practical use of the proposed method in abnormal time-series detection problems.
%
Here, we only highlight the main results.
More details can be found in Appendix \ref{appendix:details_experiments}.

\subsection{Synthetic Data Experiments} \label{subsec:synthetic_experiments}

\textbf{Experimental setup.} We compared the SI-DTW (proposed method) with SI-DTW-oc (simple version of the proposed method that does not require parametric DTW algorithm), naive method and data splitting (DS). The details of SI-DTW-oc, naive, and DS are described in Appendix \ref{appendix:details_experiments}.

We considered the following covariance matrices:

$\bullet$ Independence: $\Sigma_{\bm X} = I_n$,   $\Sigma_{\bm Y} = I_m$.

$\bullet$ Correlation: $\Sigma_{\bm X} = \left [0.5^{{\rm abs}(i - i^\prime)} \right ]_{ii^\prime} \in \RR^{n \times n}$,  $\Sigma_{\bm Y} =\left [0.5^{{\rm abs}(j - j^\prime)} \right ]_{jj^\prime} \in \RR^{m \times m}$.

We generated $\bm X$ and $\bm Y$ with $\bm \mu_{\bm X} = \bm 0_n$, $\bm \mu_{\bm Y} = \bm 0_m + \Delta$ (element-wise addition), $\bm \veps_{\bm X} \sim \NN (\bm 0_n, \Sigma_{\bm X})$, and $\bm \veps_{\bm Y} \sim \NN (\bm 0_m, \Sigma_{\bm Y})$.
Regarding the experiments of FPR and coverage properties of the confidence interval (CI), we set $\Delta = 0$, $n = 10$, and ran 120 trials for each $m \in \{ 10, 20, 30, 40\} $.
In regard to the experiments of true positive rate (TPR) and CI length,  we set $n = 10$, $m = 20$, and ran 120 trials for each $\Delta \in \{ 2, 3, 4, 5\} $. 
We set the significance level $\alpha = 0.05$ and $\tau = 2.0$.

%\begin{figure}[!t]
%\begin{minipage}{.495\linewidth}
%\begin{subfigure}{.495\linewidth}
%  \centering
%  \includegraphics[width=\linewidth]{tpr_comparison_independence}  
%  \caption{Case 1  (independence)}
%\end{subfigure}
%\begin{subfigure}{.495\linewidth}
%  \centering
%  \includegraphics[width=\linewidth]{tpr_comparison_correlation} 
%  \caption{Case 2  (correlation)}
%\end{subfigure}
%\caption{TPR comparison} 
%\label{fig:tpr}
%\end{minipage}
%%
%\hspace{3pt}
%%
%\begin{minipage}{.495\linewidth}
%\vspace{-20pt}
%\begin{table}[H]
%\renewcommand{\arraystretch}{1.1}
%{\footnotesize
%\centering
%\caption{Results on real-world datasets}
%\begin{tabular}{ |l|c|c|c|c| } 
%  \hline
%  & \multicolumn{2}{|c|}{$N = 240$} & \multicolumn{2}{|c|}{$N = 480$} \\ 
%  \hline
%  & FPR & TPR & FPR & TPR \\
%  \hline
%   \hline
%  \multicolumn{5}{|c|}{\textbf{Results on heart beat dataset}} \\ 
%  \hline
% \textbf{SI-DTW-oc} & 0.042 & 0.375 & 0.038 & 0.400 \\ 
%  \hline
% \textbf{SI-DTW} & 0.033 & \textbf{0.708} & 0.042 & \textbf{0.717}\\ 
% \hline
%  \multicolumn{5}{|c|}{\textbf{Results on respiration dataset}} \\ 
%  \hline
% \textbf{SI-DTW-oc} & 0.033 & 0.217 & 0.038 & 0.196 \\ 
%  \hline
% \textbf{SI-DTW} & 0.042 & \textbf{0.883} & 0.046 & \textbf{0.879}\\ 
% \hline
%\end{tabular}
%\label{tbl:real_data_experiments}
%%\vspace{-20pt}
%}
%\end{table}
%\end{minipage}
%\end{figure}

\begin{table}[!t]
\centering
\caption{Results on heart beat dataset}
\vspace{2pt}
\begin{tabular}{ |l|c|c|c|c| } 
  \hline
  & \multicolumn{2}{|c|}{$N = 240$} & \multicolumn{2}{|c|}{$N = 480$} \\ 
  \hline
  & FPR & TPR & FPR & TPR \\
  \hline
  \hline
 \textbf{Naive} & 0.23 & N/A & 0.21 & N/A \\ 
 \hline
 \textbf{DS} & 0.07 & N/A & 0.08 &  N/A \\ 
  \hline
 \textbf{SI-DTW-oc} & \textbf{0.04} & 0.38 & \textbf{0.04} & 0.40 \\ 
  \hline
 \textbf{SI-DTW} & \textbf{0.03} & \textbf{0.71} & \textbf{0.04} & \textbf{0.72}\\ 
 \hline
\end{tabular}
\label{tbl:heart_beat}
\vspace{-5pt}
\end{table}

\begin{table}[!t]
\centering
\caption{Results on respiration dataset}
\vspace{2pt}
\begin{tabular}{ |l|c|c|c|c| } 
  \hline
  & \multicolumn{2}{|c|}{$N = 240$} & \multicolumn{2}{|c|}{$N = 480$} \\ 
  \hline
  & FPR & TPR & FPR & TPR \\
  \hline
  \hline
 \textbf{Naive} & 0.60  & N/A & 0.52 & N/A \\ 
 \hline
 \textbf{DS} & 0.12 & N/A & 0.13 & N/A \\ 
  \hline
 \textbf{SI-DTW-oc} & \textbf{0.03} & 0.22 & \textbf{0.04} & 0.20 \\ 
  \hline
 \textbf{SI-DTW} & \textbf{0.04} & \textbf{0.89} & \textbf{0.05} & \textbf{0.88}\\ 
 \hline
\end{tabular}
\label{tbl:respiration}
\vspace{-5pt}
\end{table}

\textbf{Numerical Result.} The results of the FPR control and coverage guarantee of CI are shown in Fig. \ref{fig:fpr_ci_coverage}.
The SI-DTW and SI-DTW-oc successfully controlled the FPR under $\alpha = 0.05$  as well as guaranteeing the $95\%$ coverage property of the CI in both cases of independence and correlation whereas the naive method and DS \emph{could not}. 
Because the naive method and DS failed to control the FPR, we no longer considered the TPR and CI length.
The TPR results are shown in Fig. \ref{fig:tpr}.
The SI-DTW has higher TPR than the SI-DTW-oc in all the cases.
Due to the space limitation, we deferred the results on CI length to Appendix \ref{appendix:details_experiments}.
%The results on CI length are shown in Fig. \ref{fig:ci_length}.
%
In general, the TPR results are consistent with the results on CI length, i.e.,
the SI-DTW has higher TPR than SI-DTW-oc which indicates it has shorter CI.
Additionally, we conducted the experiments on computational time and the robustness of the proposed method.
% in terms of the FPR control and coverage of the CI.
%
The details are provided in Appendix \ref{appendix:details_experiments}.

\subsection{Real-data Examples} \label{subsec:real_data_experiments}

We consider two settings to demonstrate how the $p$-value of the DTW distance can be used in data analysis tasks.
In the first setting, we consider an abnormal time-series detection problem for heart-beat signals and respiration signals where the signals were generated by a generator called 
NeuroKit2 \cite{Makowski2021neurokit}.
In the second setting, we used six benchmark datasets: 
Italy Power Demand, Melbourne Pedestrian,
Smooth Subspace, EEG Eye State, China Town,
and Finger Movement. 
Each dataset contains two classes of time-series. 
The details of the datasets are in Appendix \ref{appendix:details_real_data}.

In our experiments, we picked the ``reference'' time series as follows.
Given a set of normal time-series, we randomly choose one time series from this set and designate it as the reference time series for each run.
We also used an independent set of normal time-series for estimating $\Sigma_{\bm X}$ and $\Sigma_{\bm Y}$ by using empirical variance.

\textbf{Setting 1.} We considered the abnormal time-series detection on heart beat and respiration datasets.
The goal is to test if the new query time-series is normal or abnormal, based on the $p$-value of the DTW distance between the query and reference time-series.
Here, we conducted the comparisons for $N \in \{240, 480\}$ ($N / 2$ normal time-series and $N/2$ abnormal time-series).
The results are shown in Tabs. \ref{tbl:heart_beat} and \ref{tbl:respiration}.
%The results are shown in Tab. \ref{tbl:real_data_experiments}.
%
Because the naive method and DS cannot properly control the FPR under $\alpha=0.05$, a comparison of TPR is no longer needed.
While both SI-DTW-oc and SI-DTW could control the FPR, the SI-DTW method had higher TPR than the SI-DTW-oc in all the cases.

%\begin{table}[!t]
%%\renewcommand{\arraystretch}{1.1}
%\centering
%\caption{Results on heart beat dataset}
%\vspace{-5pt}
%\begin{tabular}{ |l|c|c|c|c| } 
%  \hline
%  & \multicolumn{2}{|c|}{$N = 240$} & \multicolumn{2}{|c|}{$N = 480$} \\ 
%  \hline
%  & FPR & TPR & FPR & TPR \\
%  \hline
%  \hline
% \textbf{SI-DTW-oc} & 0.042 & 0.375 & 0.038 & 0.400 \\ 
%  \hline
% \textbf{SI-DTW} & 0.033 & \textbf{0.708} & 0.042 & \textbf{0.717}\\ 
% \hline
%\end{tabular}
%\label{tbl:heart_beat}
%\vspace{-5pt}
%\end{table}
%
%\begin{table}[!t]
%%\renewcommand{\arraystretch}{1.1}
%\centering
%\caption{Results on respiration dataset}
%\vspace{-5pt}
%\begin{tabular}{ |l|c|c|c|c| } 
%  \hline
%  & \multicolumn{2}{|c|}{$N = 240$} & \multicolumn{2}{|c|}{$N = 480$} \\ 
%  \hline
%  & FPR & TPR & FPR & TPR \\
%  \hline
%  \hline
% \textbf{SI-DTW-oc} & 0.033 & 0.217 & 0.038 & 0.196 \\ 
%  \hline
% \textbf{SI-DTW} & 0.042 & \textbf{0.883} & 0.046 & \textbf{0.879}\\ 
% \hline
%\end{tabular}
%\label{tbl:respiration}
%\vspace{-10pt}
%\end{table}

\begin{figure}[!t]
\begin{subfigure}{.495\linewidth}
  \centering
  \includegraphics[width=\linewidth]{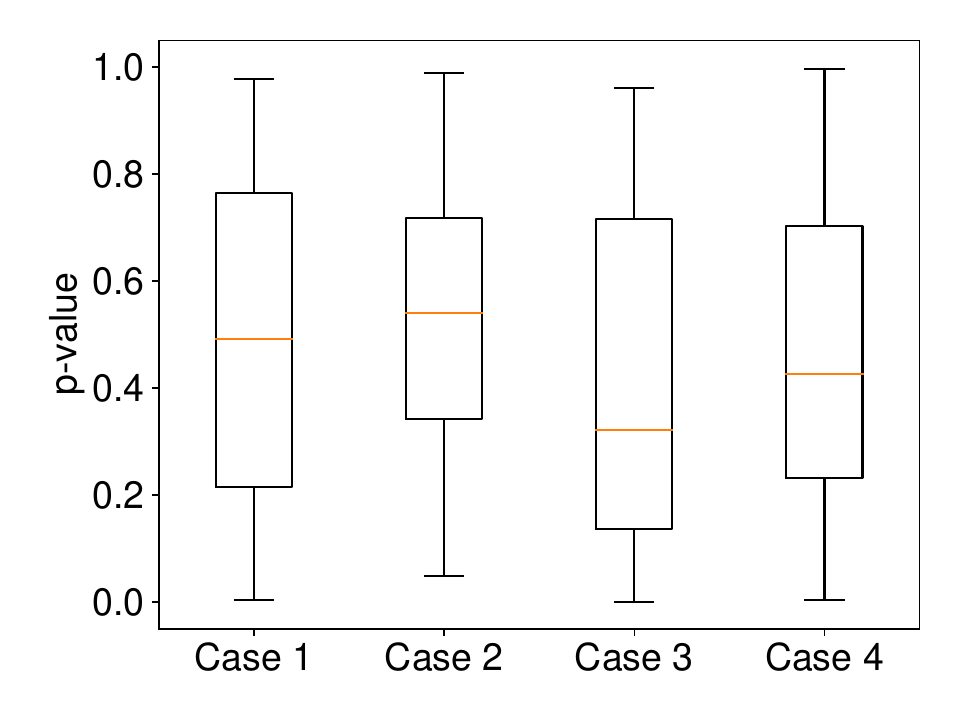}  
  \caption{Italy}
\end{subfigure}
\begin{subfigure}{.495\linewidth}
  \centering
  \includegraphics[width=\linewidth]{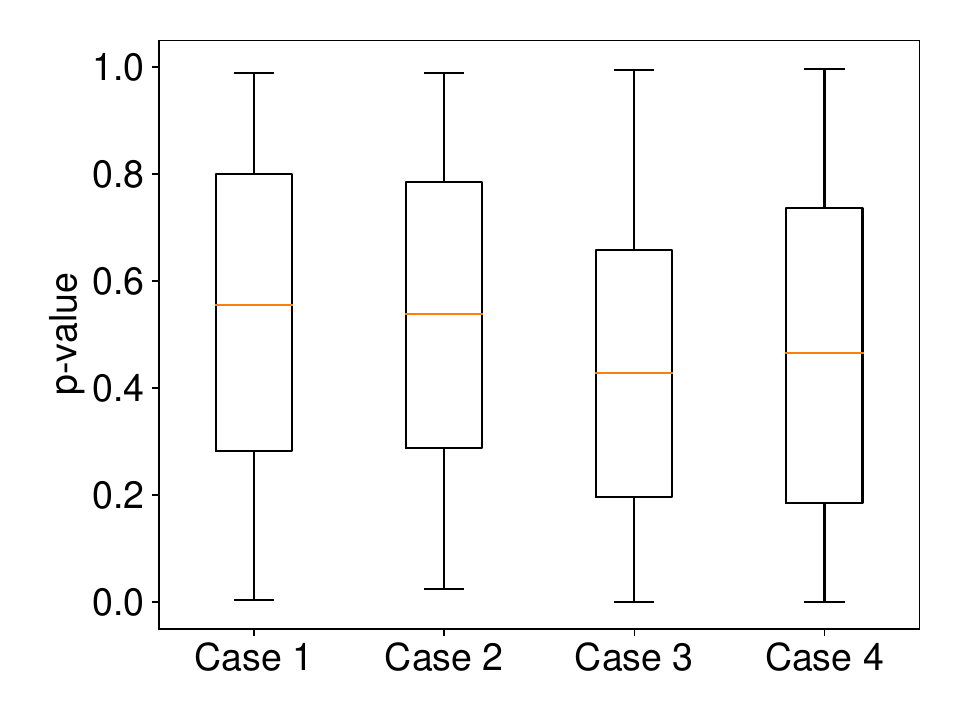}  
  \caption{Melbourne}
\end{subfigure}
\begin{subfigure}{.495\linewidth}
  \centering
  \includegraphics[width=\linewidth]{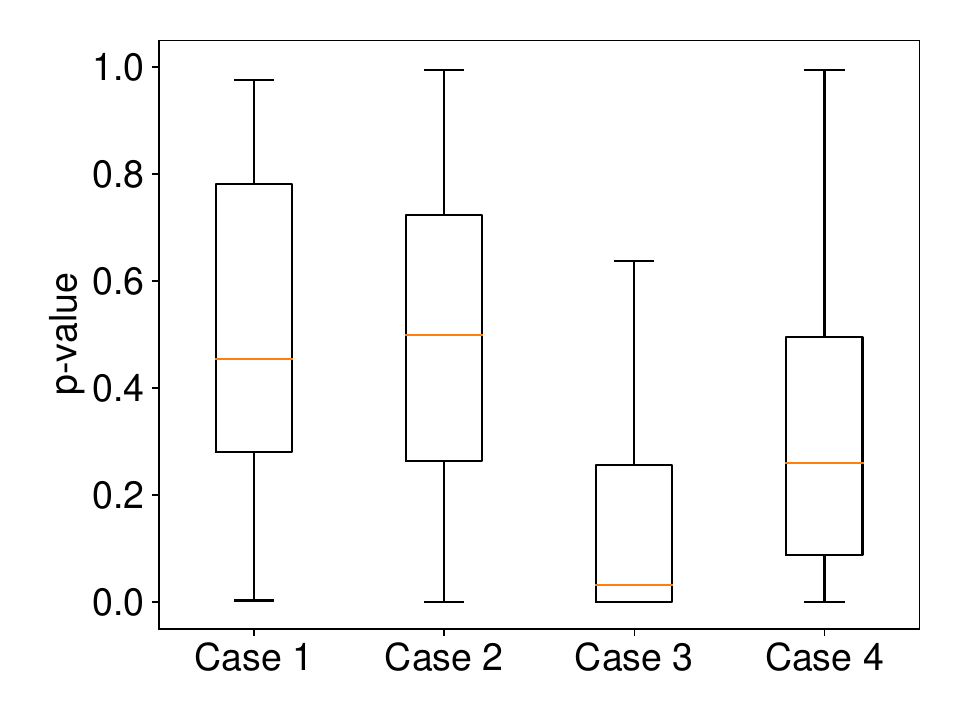}  
  \caption{Smooth Subspace}
 \end{subfigure}
\begin{subfigure}{.495\linewidth}
  \centering
  \includegraphics[width=\linewidth]{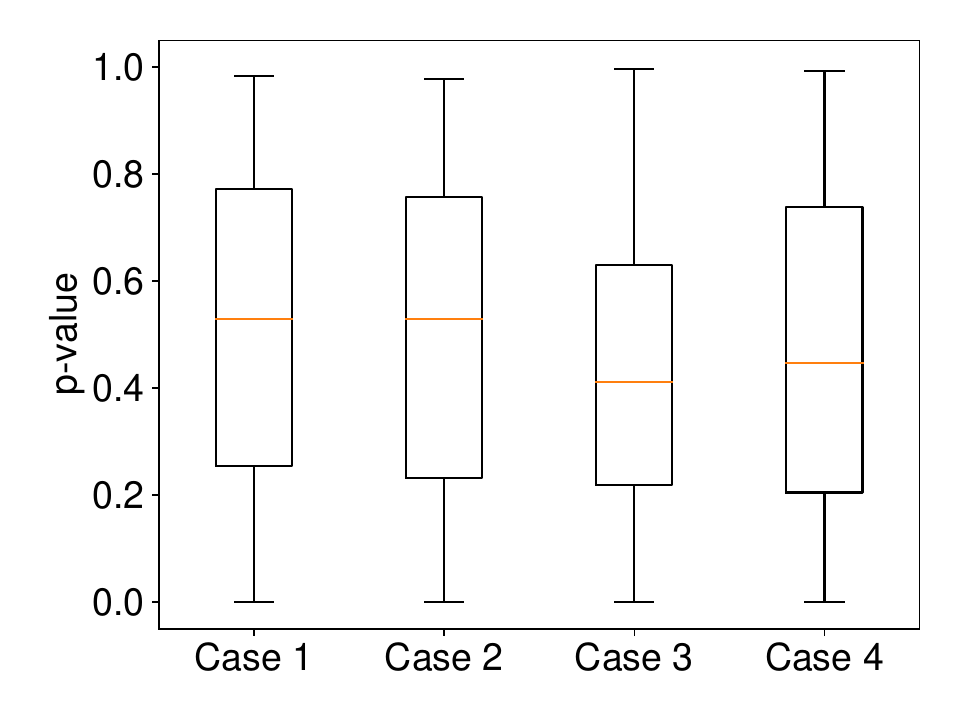}  
  \caption{EEG Eye State}
\end{subfigure}
\begin{subfigure}{.495\linewidth}
  \centering
  \includegraphics[width=\linewidth]{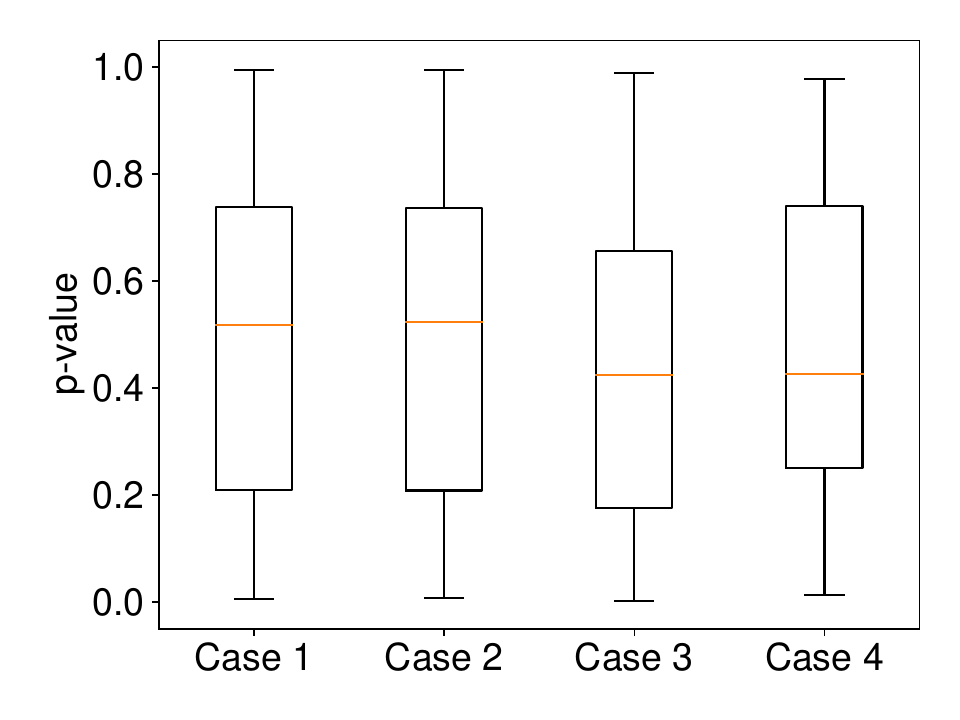}  
  \caption{China Town}
\end{subfigure}
\begin{subfigure}{.495\linewidth}
  \centering
  \includegraphics[width=\linewidth]{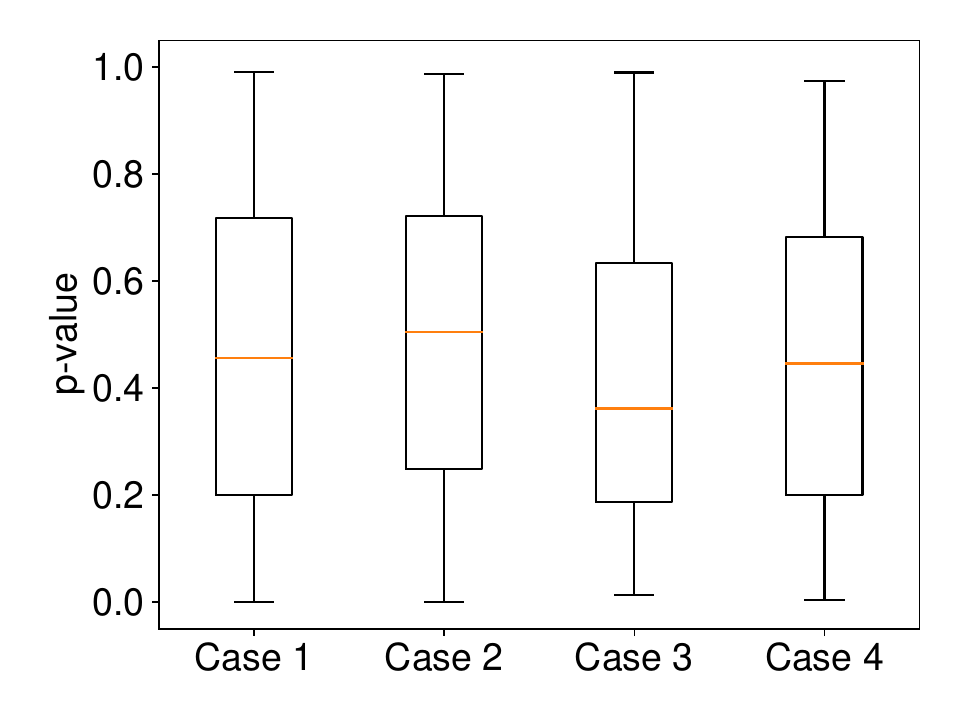}  
  \caption{Finger Movement}
 \end{subfigure}
\caption{Boxplots of the distribution of the $p$-values.}
\label{fig:boxplot_p_value}
\vspace{-15pt}
\end{figure}

\textbf{Setting 2.}
For each of the six datasets, we present the distributions of the $p$-values in four cases:

$\bullet$ Case 1:  the $p$-values of the SI-DTW when two time-series are sampled from the same class,

$\bullet$ Case 2:  the $p$-values of the SI-DTW-oc  when two time-series are sampled from the same class,

$\bullet$ Case 3: the $p$-values of the SI-DTW  when two time-series are  sampled from different classes,

$\bullet$ Case 4: the $p$-values of the SI-DTW-oc  when two time-series are  sampled from different classes.

If the two time-series are from the same class, it can be seen as a situation in which both the query and reference time-series are normal. 
If the two time-series are from different classes, it can be viewed as a case where the time-series from the first class is an abnormal query and the time-series from the second class is a normal reference time-series.
% \footnote{In setting 2, the time series within the same class may not be identical samples from the same distribution. Therefore, it is reasonable that the FPR may not be less than or equal to the $\alpha$.}.
%
In the experiments, we randomly selected pairs of time-series for each time of running SI-DTW and SI-DTW-oc.
Fig. \ref{fig:boxplot_p_value} shows the boxplots of the distribution of the $p$-values. % in the four cases. 
We compare the performance between SI-DTW and SI-DTW-oc methods (i.e., Case 1 vs. Case 2 and Case 3 vs. Case 4).
The $p$-values of the former tend to be smaller than those of the latter.
This is because the truncation region of SI-DTW tends to be larger than that of SI-DTW-oc, i.e., we have more information for conducting inference in SI-DTW compared to SI-DTW-oc.
The results indicate that the SI-DTW method is more powerful than the SI-DTW-oc.
%
%In regard to the comparison between the cases where two time-series are sampled from the same class or different classes (i.e., Case 1 vs. Case 3 and Case 2 vs. Case 4), the $p$-values of the latter tend to be smaller than those of the former.
%%
%This suggests that the DTW distance between the two time-series from different classes tend to be more statistically significant than the ones from the same class.

%% file: sec6.tex
\section{Conclusion} \label{sec:conclusion}

We present a valid inference method for the DTW distance between two time-series.
This is the first method that can provide valid $p$-values and confidence intervals for the DTW distance.
%
%We conducted several experiments to show the good performance of the proposed method.
%
%We believe this study is an important contribution toward reliable Machine Learning (ML), which is one of the most critical issues in the ML community.
We believe this study is an important contribution in introducing a new aspect of statistical reliability in the literature of time-series data.
Some open questions remain:

$\bullet$ The proposed method currently can only handle the case in which the test-statistic is a linear contrast w.r.t the data. 
Therefore, an extension to quadratic test-statistics could be a potential future direction.

$\bullet$ There are several variants of the DTW distance, such as Soft-DTW or FastDTW, which have been proposed in the literature for the purpose of reducing computation time of the DTW. 
Thus, extensions of the proposed framework to these variants would also stand as a valuable future contribution.

%% file: appendix.tex
\section{Appendix} \label{sec:appx}

\subsection{Examples of $C_{\rm vec}(\mathbf{X}, \mathbf{Y})$, $\Omega$ and $\hat{M}_{\rm vec}$} \label{appendix:examples}

Given $\bm X =  (x_1, x_2)^\top$ and $\bm Y = (y_1, y_2)^\top$, the cost matrix is  
\begin{align*}
	C(\bm X, \bm Y) = 
	\begin{pmatrix}
		(x_1 - y_1)^2 & (x_1 - y_2)^2 \\ 
		(x_2 - y_1)^2 & (x_2 - y_2)^2
	\end{pmatrix}.
\end{align*}
Then, we have 
\begin{align*}
	C_{\rm vec} (\bm X, \bm Y) 
	= 
	\begin{pmatrix}
		(x_1 - y_1)^2 \\ 
		(x_1 - y_2)^2 \\ 
		(x_2 - y_1)^2 \\
		(x_2 - y_2)^2
	\end{pmatrix}
	= 
	\Omega
	\begin{pmatrix}
		x_1 \\ 
		x_2 \\ 
		y_1 \\
		y_ 2
	\end{pmatrix}
	\circ
	\Omega
	\begin{pmatrix}
		x_1 \\ 
		x_2 \\ 
		y_1 \\
		y_ 2
	\end{pmatrix},
\end{align*}
where
$
	\Omega = 
	\begin{pmatrix}
		1 & 0 & - 1 & 0 \\ 
		1 & 0 & 0 & - 1 \\ 
		0 & 1 & - 1 & 0 \\ 
		0 & 1 & 0 & - 1 
	\end{pmatrix}
$.
Similarly, given 
$
\hat{M} = 
\begin{pmatrix}
	1 & 0 \\ 
	0 & 1
\end{pmatrix}
$,
then 
$
\hat{M}_{\rm vec} 
= 
\begin{pmatrix}
	1 & 0 & 0 & 1
\end{pmatrix}^\top 
$.

\subsection{Proof of Lemma \ref{lemma:valid_selective_p}} \label{appx:proof_valid_selective_p}
We have 
\[
	\hspace{-6pt}
	\bm \eta_{\hat{M}, \hat{\bm s}}^\top {\bm X \choose \bm Y} \mid 
	\left \{ 
		\cA(\bm X, \bm Y) = \hat{M}^{\rm obs},
		\cS(\bm X, \bm Y) = \hat{\bm s}^{\rm obs},
		\cQ (\bm X, \bm Y) = \hat{\bm q}^{\rm obs}
	\right \}
	\sim 
	{\rm TN} 
	\left (
	\bm \eta_{\hat{M}, \hat{\bm s}}^\top 
	{\bm \mu_{\bm X} \choose \bm \mu_{\bm Y}},
	\bm \eta_{\hat{M}, \hat{\bm s}}^\top \Sigma \bm \eta_{\hat{M}, \hat{\bm s}},
	\cZ
	\right ),
\]
which is a truncated normal distribution with a mean $\bm \eta_{\hat{M}, \hat{\bm s}}^\top 
	{\bm \mu_{\bm X} \choose \bm \mu_{\bm Y}}$,
 variance $\bm \eta_{\hat{M}, \hat{\bm s}}^\top \Sigma \bm \eta_{\hat{M}, \hat{\bm s}}$, and the truncation region $\cZ$ described in \eq{eq:cZ}.
Therefore, under the null hypothesis,
\begin{align*}
	\mathbb{P}_{\rm H_0}  \Big (
	p_{\rm sel} \leq \alpha \mid 
	\cA(\bm X, \bm Y) = \hat{M}^{\rm obs},
	\cS(\bm X, \bm Y) = \hat{\bm s}^{\rm obs},
	\cQ (\bm X, \bm Y) = \hat{\bm q}^{\rm obs}
	\Big ) \leq \alpha, \quad \forall \alpha \in [0, 1].
\end{align*}

Next, we have 
\begin{align*}
	&\mathbb{P}_{\rm H_0}  \Big (
	p_{\rm sel} \leq \alpha \mid 
	\cA(\bm X, \bm Y) = \hat{M}^{\rm obs},
	\cS(\bm X, \bm Y) = \hat{\bm s}^{\rm obs}
	\Big ) \\ 
	& = \int \mathbb{P}_{\rm H_0}  \left (
	p_{\rm sel} \leq \alpha ~ \Bigg | 
	\begin{array}{l}
	\cA(\bm X, \bm Y) = \hat{M}^{\rm obs}, \\ 
	\cS(\bm X, \bm Y) = \hat{\bm s}^{\rm obs}, \\ 
	\cQ (\bm X, \bm Y) = \hat{\bm q}^{\rm obs}
	\end{array}
	\right )
	\mathbb{P}_{\rm H_0} \left (
		\cQ (\bm X, \bm Y) = \hat{\bm q}^{\rm obs}
		~ \Big |  
		\begin{array}{l}
		\cA(\bm X, \bm Y) = \hat{M}^{\rm obs}, \\ 
		\cS(\bm X, \bm Y) = \hat{\bm s}^{\rm obs}
		\end{array}
	\right )
	d\cQ (\bm X, \bm Y)
	\\ 
	& \leq \int \alpha ~ \mathbb{P}_{\rm H_0} \left (
		\cQ (\bm X, \bm Y) = \hat{\bm q}^{\rm obs}
		~ \Big |  
		\begin{array}{l}
		\cA(\bm X, \bm Y) = \hat{M}^{\rm obs}, \\ 
		\cS(\bm X, \bm Y) = \hat{\bm s}^{\rm obs}
		\end{array}
	\right )
	d\cQ (\bm X, \bm Y) \\ 
	& = \alpha
	\int
	\mathbb{P}_{\rm H_0} \left (
		\cQ (\bm X, \bm Y) = \hat{\bm q}^{\rm obs}
		~ \Big |  
		\begin{array}{l}
		\cA(\bm X, \bm Y) = \hat{M}^{\rm obs}, \\ 
		\cS(\bm X, \bm Y) = \hat{\bm s}^{\rm obs}
		\end{array}
	\right )
	d\cQ (\bm X, \bm Y) \\ 
	& = \alpha.
\end{align*}
Finally, we obtain the result of Lemma \ref{lemma:valid_selective_p} as follows:
\begin{align*}
	 \mathbb{P}_{\rm H_0}  \Big (
	p_{\rm sel} \leq \alpha
	\Big ) 
	& = 
	\sum \limits_{\big (\hat{M}^{\rm obs}, \hat{\bm s}^{\rm obs} \big )}
	\mathbb{P}_{\rm H_0}  \left (
	p_{\rm sel} \leq \alpha 
	~ \Big | 
	\begin{array}{l}
	\cA(\bm X, \bm Y) = \hat{M}^{\rm obs}, \\ 
	\cS(\bm X, \bm Y) = \hat{\bm s}^{\rm obs}
	\end{array}
	\right )
	\mathbb{P}_{\rm H_0}  \Big (
	\cA(\bm X, \bm Y) = \hat{M}^{\rm obs}, \cS(\bm X, \bm Y) = \hat{\bm s}^{\rm obs}
	\Big )
	\\ 
	& \leq
	\sum \limits_{\big (\hat{M}^{\rm obs}, \hat{\bm s}^{\rm obs} \big )}
	\alpha ~ 
	\mathbb{P}_{\rm H_0}  \Big (
	\cA(\bm X, \bm Y) = \hat{M}^{\rm obs}, \cS(\bm X, \bm Y) = \hat{\bm s}^{\rm obs}
	\Big ) \\ 
	& = \alpha \sum \limits_{\big (\hat{M}^{\rm obs}, \hat{\bm s}^{\rm obs} \big )}
	\mathbb{P}_{\rm H_0}  \Big (
	\cA(\bm X, \bm Y) = \hat{M}^{\rm obs}, \cS(\bm X, \bm Y) = \hat{\bm s}^{\rm obs}
	\Big ) \\ 
	& = \alpha.
\end{align*}

\subsection{Selective Confidence Interval} \label{appendix:selective_ci}

Similar to the computation of the selective $p$-value, we can also compute the selective confidence interval $C_{\rm sel}$ of the DTW distance that satisfies the following $(1 - \alpha)$-coverage property:
\begin{align} \label{eq:ci_selective_validity}
	\PP\left( W^\ast \in C_{\rm {sel}}  
	\mid
	\cA(\bm X, \bm Y) = \hat{M}^{\rm obs}, ~
	\cS(\bm X, \bm Y) = \hat{\bm s}^{\rm obs}
	\right ) = 1 - \alpha,
\end{align}
for any $\alpha \in [0, 1]$. 
The selective CI is defined as 
\begin{align} \label{eq:ci_selective}
	C_{\rm {sel}} = 
	\left \{ 
	w \in \RR : 
	\frac{\alpha}{2} \leq
	F_{w, \sigma^2}^{\cZ} 
	\left (
	\bm \eta_{\hat{M}, \hat{\bm s}}^\top
	{\bm X^{\rm {obs}} \choose \bm Y^{\rm {obs}} }
	\right )
	\leq 1 - \frac{\alpha}{2} 
	\right \},
\end{align}
where the quantity 
\begin{align}
F_{w, \sigma^2}^{\cZ}
	\left (
	\bm \eta_{\hat{M}, \hat{\bm s}}^\top {\bm X \choose \bm Y}
	\right )
	\mid 
	\left \{ 
	\begin{array}{l}
	\cA(\bm X, \bm Y) = \hat{M}^{\rm obs}, 
	\cS(\bm X, \bm Y) = \hat{\bm s}^{\rm obs}, 
	\cQ (\bm X, \bm Y) = \hat{\bm q}^{\rm obs}
	\end{array} 
	\right \}  
\end{align}
is the c.d.f of the \emph{truncated} normal distribution with a mean 
$w \in \RR$, 
variance 
$\sigma^2 = \bm \eta_{\hat{M}, \hat{\bm s}}^\top 
\begin{pmatrix}
	\Sigma_{\bm X} & 0 \\ 
	0 & \Sigma_{\bm Y}
\end{pmatrix}
\bm \eta_{\hat{M}, \hat{\bm s}}^\top$,
and truncation region $\cZ$.

%========

\subsection{Proof of Lemma \ref{lemma:data_line}} \label{appendix:proof_lemma_data_line}

According to the third condition in \eq{eq:conditional_data_space}, we have 
\begin{align*}
	\cQ (\bm X, \bm Y) & =  \hat{\bm q}^{\rm obs} \\ 
	\Leftrightarrow 
	\Big ( I_{n+m} - 
	\bm b
	\bm \eta_{\hat{M}, \hat{\bm s}}^\top \Big ) 
	( \bm X ~ \bm Y )^\top
	& = 
	\hat{\bm q}^{\rm obs}\\ 
	\Leftrightarrow 
	( \bm X ~ \bm Y )^\top
	& = 
	\hat{\bm q}^{\rm obs}
	+ \bm b
	\bm \eta_{\hat{M}, \hat{\bm s}}^\top  
	( \bm X ~ \bm Y )^\top.
\end{align*}
By defining 
$\bm a = \hat{\bm q}^{\rm obs}$,
$z = \bm \eta_{\hat{M}, \hat{\bm s}}^\top  \Big ( \bm X ~ \bm Y \Big )^\top$, and incorporating the first and second conditions in \eq{eq:conditional_data_space}, we obtain the results in Lemma \ref{lemma:data_line}.

\subsection{Proof of Lemma \ref{lemma:bellman_parametric}} \label{appendix:proof_lemma_bellman_parametric}

We prove the lemma by showing that any alignment matrix that is NOT in \begin{align*}
\hat{\cM}_{i - 1, j} ~ \bigcup ~ \hat{\cM}_{i, j - 1} ~ \bigcup ~  \hat{\cM}_{i - 1, j - 1}
\end{align*}
will never be a sub-matrix of the optimal alignment matrices in larger problem with $i$ and $j$ for any $z \in \RR$. 
Let $\RR^{(i - 1) \times j} \ni M \not \in \hat{\cM}_{i - 1, j}$ be the alignment matrix that is NOT optimal for all 
$z \in \RR$,
i.e.,
\begin{align*}
 L_{i - 1, j}(M, z) > \hat{L}_{i - 1, j}(z)
 ~~~
 \forall z \in \RR.
\end{align*}
It suggests that, for any $z \in \RR$ and $c_{ij}(z) = \big (\bm X_i(z) - \bm Y_i(z) \big)^2$,
\begin{align*}
	L_{i - 1, j}(M, z) + c_{ij}(z)
	& > \min \limits_{\hat{M} \in \hat{\cM}_{i - 1, j}} L_{i - 1, j}(\hat{M}, z) + c_{ij}(z) \\ 
	& = \hat{L}_{i - 1, j} (z) + c_{ij}(z)\\ 
	& \geq \hat{L}_{i, j} (z).
\end{align*}
Thus, $M$ cannot be a sub-matrix of the optimal alignment matrices in larger problem with $i$ and $j$ for any $z \in \RR$. 
Similar proofs can be applied for $\RR^{i \times (j - 1)} \ni M \not \in \bigcup ~ \hat{\cM}_{i, j - 1}$  and $\RR^{(i - 1) \times (j - 1)} \ni M \not \in \bigcup ~ \hat{\cM}_{i - 1, j - 1}$.
In other words, only the alignment matrices in 
$\hat{\cM}_{i - 1, j} ~ \bigcup ~ \hat{\cM}_{i, j - 1} ~ \bigcup ~  \hat{\cM}_{i - 1, j - 1}$ can be used as the sub-matrix of optimal alignment matrices for larger problems with $i$ and $j$.

\subsection{Proof of Lemma \ref{lemma:cZ_2}} \label{appendix:proof_lemma_cZ_2}

Let us first remind that 
$
	\hat{\bm s} = \cS(\bm X, \bm Y) =  
	{\rm sign} 
	\left ( 
	\hat{M}_{\rm vec} 
	\circ 
	\left [ \Omega  (\bm X ~ \bm Y)^\top \right ]
	\right ),
$
which is defined in  \eq{eq:test_statistic_first}.
Then, the set $\cZ_2$ can be re-written as follows:
\begin{align*}
	\cZ_2 
	& = \{ z \in \RR \mid \cS(\bm a + \bm b z) = \hat{\bm s}^{\rm {obs}} \} \\ 
	& = \left \{ z \in \RR \mid 
	{\rm sign} 
	\left ( 
	\hat{M}_{\rm vec} 
	\circ 
	 \Omega  (\bm a + \bm b z)
	\right )
	= \hat{\bm s}^{\rm {obs}} \right \} \\ 
	& = \left \{ z \in \RR \mid 
	\hat{\bm s}^{\rm {obs}}
	\circ 
	\hat{M}_{\rm vec} 
	\circ 
	 \Omega  (\bm a + \bm b z)
	\geq \bm 0 \right \}.
\end{align*}
By defining $\bm \nu^{(1)} = \hat{\bm s}^{\rm obs}  \circ \hat{M}_{\rm vec} \circ \Omega \bm a$
and  
$\bm \nu^{(2)} = \hat{\bm s}^{\rm obs}  \circ \hat{M}_{\rm vec} \circ \Omega \bm b$, the result of Lemma \ref{lemma:cZ_2} is straightforward
by solving the above system of linear inequalities.

%========

\subsection{More details of Algorithm \ref{alg:paraOptAlign}} \label{appendix:details_algorithm_paraOptAlign}

The algorithm is initialized at the optimal alignment matrix
for $z_1 = -\infty$, which can be easily identified based on the coefficients of the QFs. 
At step
$t, t \in [\cT],$
the task is to find the next breakpoint
$z_{t + 1}$
and the next optimal alignment matrix
$\hat{M}_{t + 1}$.
This task can be done by finding the smallest $z_{t + 1}$
such that
$z_{t + 1} > z_t$
among the intersections of the current QF
$L_{n, m} \big (\hat{M}_t, z \big )$
and
each of the other QFs
$L_{n, m} (M, z)$
for
$M \in \cM_{n, m} \setminus \big \{\hat{M}_t \big \}$. 
This step is repeated until we find the optimal alignment matrix when $z_t = +\infty$.
%which can be identified based on the coefficients of the QFs.
%
The algorithm returns the sequences of the optimal alignment matrices $ \{\hat{M}_t  \}_{t=1}^{\cT - 1}$ and breakpoints $ \{z_t  \}_{t=1}^{\cT}$. 
The entire path of optimal alignment matrices for $z \in \RR$ 
is given by
\begin{align*}
 \hat{M}_{n, m}(z)
 =
 \mycase{
 \hat{M}_1 & \text{ if } z \in (z_1 = -\infty, z_2], \\
 \hat{M}_2 & \text{ if } z \in [z_2, z_3], \\
 ~~ \vdots & \\
 \hat{M}_{\cT - 1} & \text{ if } z \in [z_{\cT - 1}, z_{\cT} = +\infty).
 }
\end{align*}

At Line 2 of the Algorithm \ref{alg:paraOptAlign}, the optimal alignment matrix $\hat{M}_t$ at $z_t = - \infty$ is identified as follows.
For each $M \in \cM_{n, m}$,  the corresponding loss function is written as a positive definite quadratic function.
Therefore, at $z_t = - \infty$, the optimal alignment matrix is the one whose corresponding loss function $L_{n, m} (M, z_t)$ has the smallest coefficient of the quadratic term.
If there are more than one quadratic function having the same smallest quadratic coefficient, we then choose the one that has the largest coefficient in the linear term. If those quadratic functions still have the same largest linear coefficient, we finally choose the one that has the smallest constant term.
At Line 4 of the Algorithm \ref{alg:paraOptAlign}, since both $L_{n, m}(\hat{M}_t, z_{t+1})$ and $L_{n, m}(\hat{M}_{t + 1}, z_{t+1})$ are quadratic functions of $z_{t + 1}$, we can compute $z_{t + 1}$ by simply solving a quadratic equation.

%========

\subsection{Standard DTW (for a single value of $z$)} \label{appendix:review_standard_DTW}

In the standard DTW with $n$ and $m$, we use $n \times m$ table 
whose
$(i, j)^{\rm th}$ element contains $\hat{M}_{i, j}(z)$ that is 
the optimal alignment matrix for the sub-sequences
$\bm X(z)_{1:i}$ and $\bm Y(z)_{1:j}$. 
The optimal alignment matrix $\hat{M}_{i, j}(z)$ for each of the sub-problem with $i$ and $j$ can be used for efficiently computing the optimal alignment matrix $\hat{M}_{n, m}(z)$ for the original problem with $n$ and $m$. 
It is well-known that the following equation, which is often called \emph{Bellman equation}, holds:
\begin{align} \label{eq:bellman_standard_cost}
	c_{ij}(z) &= \big (\bm X_i(z) - \bm Y_j(z) \big)^2 \nonumber \\ 
	\hat{L}_{i, j} (z) &= c_{ij}(z) + \min \left \{ \hat{L}_{i - 1, j} (z), ~ \hat{L}_{i, j - 1} (z), ~ \hat{L}_{i - 1, j - 1} (z)\right \}.
\end{align}
Equivalently, we have 
\begin{align} \label{eq:bellman_standard_opt_matrix}
	\hat{M}_{i, j}(z) = \argmin \limits_{M \in \tilde{\cM}_{i, j}} L_{i, j} \big (M, z \big ),
\end{align}
where 
\begin{align*}
	\tilde{\cM}_{i, j} = 
	\left \{ 
	\begin{array}{l}
		{\rm vstack} \left (\hat{M}_{i - 1, j}(z), ~ (0, ...,0, 1) \right ) \in \RR^{i \times j}, \\ 
		{\rm hstack} \left (\hat{M}_{i, j - 1}(z), ~ (0, ...,0, 1)^\top \right ) \in \RR^{i \times j}\\ 
		\begin{pmatrix}
			\hat{M}_{i - 1, j - 1}(z) & ~ 0 \\ 
			0 & ~ 1 \\ 
		\end{pmatrix}
		\in \RR^{i \times j}
	\end{array}
	\right \}, 
\end{align*}
$
i \in [n] = \{1, 2, ..., n\}, j \in [m],
$
$
\hat{M}_{0, 0} (z) = \hat{M}_{i - 1, j - 1} (z) = \emptyset
$
when $i = j = 1$,
$\hat{M}_{0, j} (z) = \emptyset$ for any $j \in [m]$,
$\hat{M}_{i, 0} (z) = \emptyset$ for any $i \in [n]$,
$\rm vstack(\cdot, \cdot)$ and $\rm hstack(\cdot, \cdot)$ are vertical stack and horizontal stack operations, respectively.
The Bellman equation \eq{eq:bellman_standard_opt_matrix} enables us to efficiently compute the optimal alignment matrix for the problem with $n$ and $m$ by using the optimal alignment matrices of its sub-problems.

\subsection{Details for Experiments} \label{appendix:details_experiments}

We executed the code on Intel(R) Xeon(R) CPU E5-2687W v4 @ 3.00GHz.

\paragraph{Methods for Comparison.} %\label{appendix:subsec_methods_comparison}

We compared our SI-DTW method with the following approaches:

\begin{itemize}
	\item SI-DTW-oc: this is our first idea of introducing conditional SI for time-series similarity using the DTW by additionally conditioning on all the operations of the DTW algorithm itself to make the problem tractable.
	Then, since the selection event of SI-DTW-oc is simply represented as a single polytope in the data space, we can apply the method in the seminal conditional SI paper \cite{lee2016exact} to compute the over-conditioning $p$-value.
	The details are shown in Appendix \ref{appendix:si_dtw_oc}.
	However, such an over-conditioning leads to a loss of statistical power \cite{lee2016exact, fithian2014optimal}.
	Later, this drawback was removed by the SI-DTW method in this paper.
	
	\item Data splitting (DS): an approach that divides the dataset in half based on even and odd indices, and uses one for computing the DTW distance and the other for inference.
	
	\item Naive: this method uses the classical $z$-test to calculate the naive $p$-value, i.e.,
	\begin{align*} 
	p_{\rm naive}  = 
	\bP_{\rm H_0} 
	\left ( 
	\bm \eta_{\hat{M}, \hat{\bm s}}^\top {\bm X \choose \bm Y} \geq 
	\bm \eta_{\hat{M}, \hat{\bm s}}^\top {\bm X^{\rm obs} \choose \bm Y^{\rm obs}}
	\right ).
\end{align*}
The naive $p$-value is computed by (wrongly) assuming that $\bm \eta_{\hat{M}, \hat{\bm s}}$ does not depend on the data.
\end{itemize}

\paragraph{Experiments on CI Length.} %\label{appendix:experiment_ci_length}

The results on CI length are shown in Fig. \ref{fig:ci_length}.

\begin{figure}[!t]
\begin{subfigure}{.245\linewidth}
  \centering
  \includegraphics[width=\linewidth]{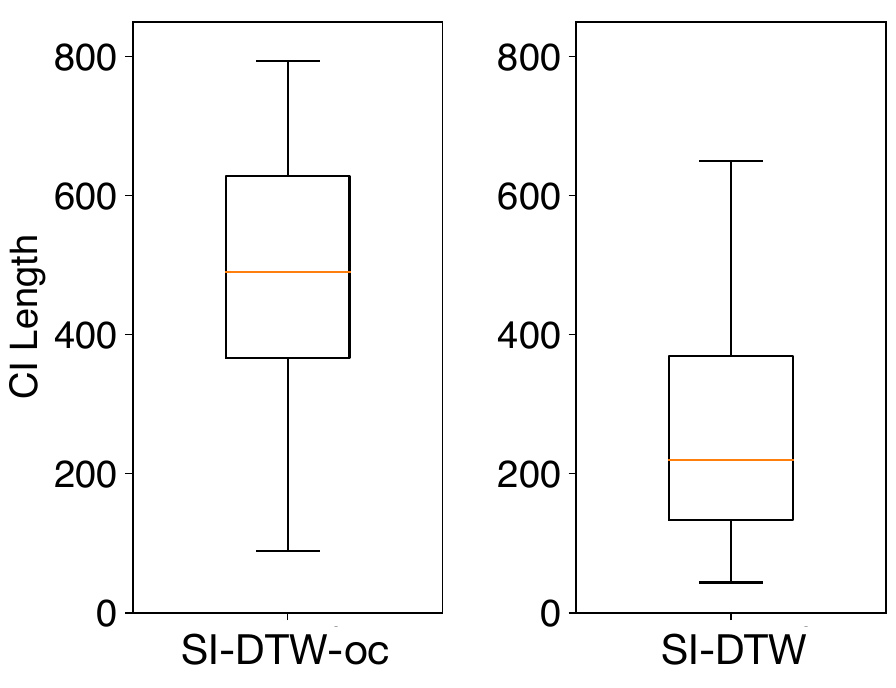}  
  \caption{$\Delta = 2$}
\end{subfigure}
\begin{subfigure}{.245\linewidth}
  \centering
  \includegraphics[width=\linewidth]{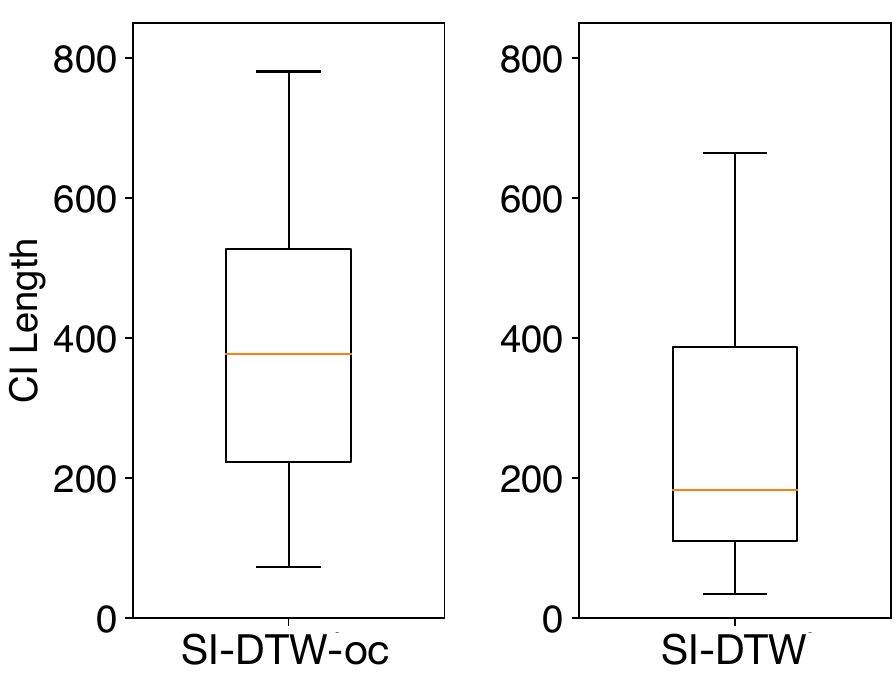} 
  \caption{$\Delta = 3$}
\end{subfigure}
\begin{subfigure}{.245\linewidth}
  \centering
  \includegraphics[width=\linewidth]{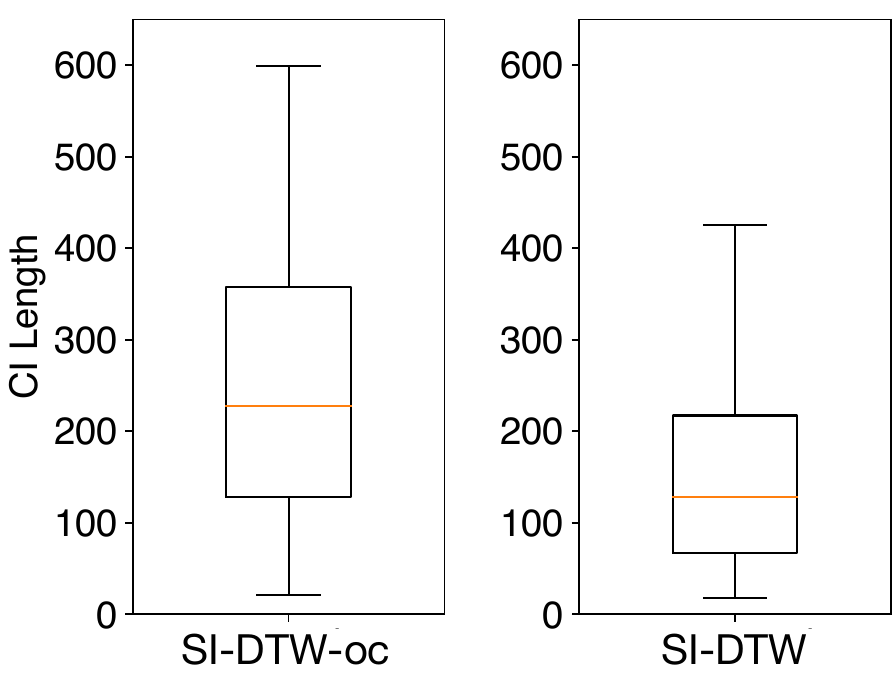}  
  \caption{$\Delta = 4$}
\end{subfigure}
\begin{subfigure}{.245\linewidth}
  \centering
  \includegraphics[width=\linewidth]{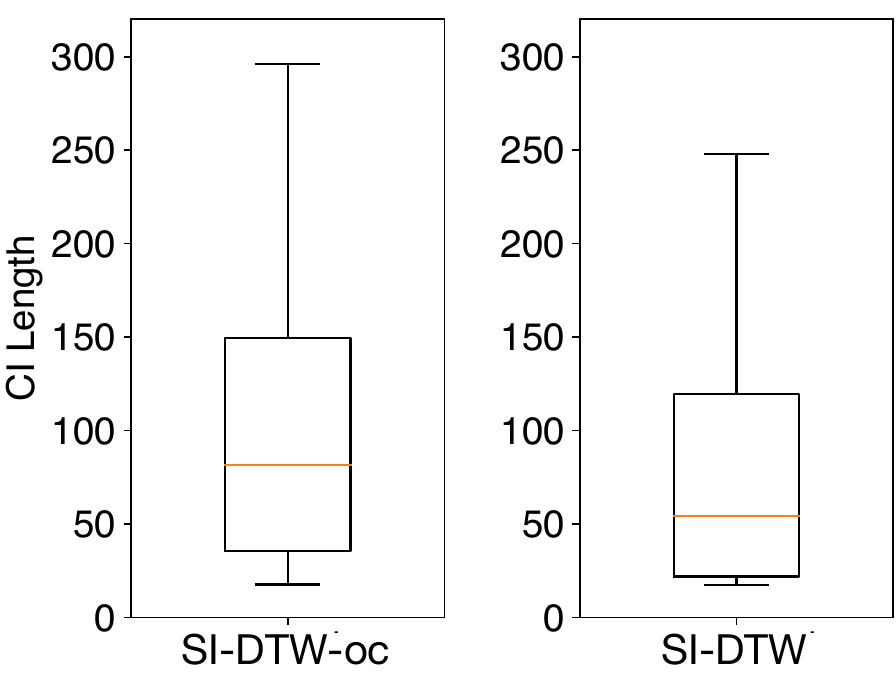} 
  \caption{$\Delta = 5$}
\end{subfigure}
\caption{CI length comparison.} 
\label{fig:ci_length}
\end{figure}

\paragraph{Experiments on Computational Time and Robustness.} %\label{appendix:subsec_robustness}

Regarding the computational time experiments, we set $n = 20$, $\Delta = 2$, and ran 10 trials for each $m \in \{ 20, 40, 60, 80\}$.
In regard to the robustness experiments, the setups were similar to the FPR experiments and we considered the following cases:

$\bullet$ Non-normal noise: the noises $\bm \veps_{\bm X}$ and $\bm \veps_{\bm Y}$ following Laplace distribution, skew normal distribution (skewness coefficient: 10), and $t_{20}$ distribution.

$\bullet$ Unknown variance: the variances of the noises were estimated from the data.

The results on computational time are shown in Fig. \ref{fig:cc}.
The results on robustness are shown in Fig. \ref{fig:robustness_1} and Fig. \ref{fig:robustness_2}.
Our method still maintains good performance on FPR control and CI coverage guarantee.

\begin{figure}[!t]
  \centering
  % include first image
  \includegraphics[width=.5\linewidth]{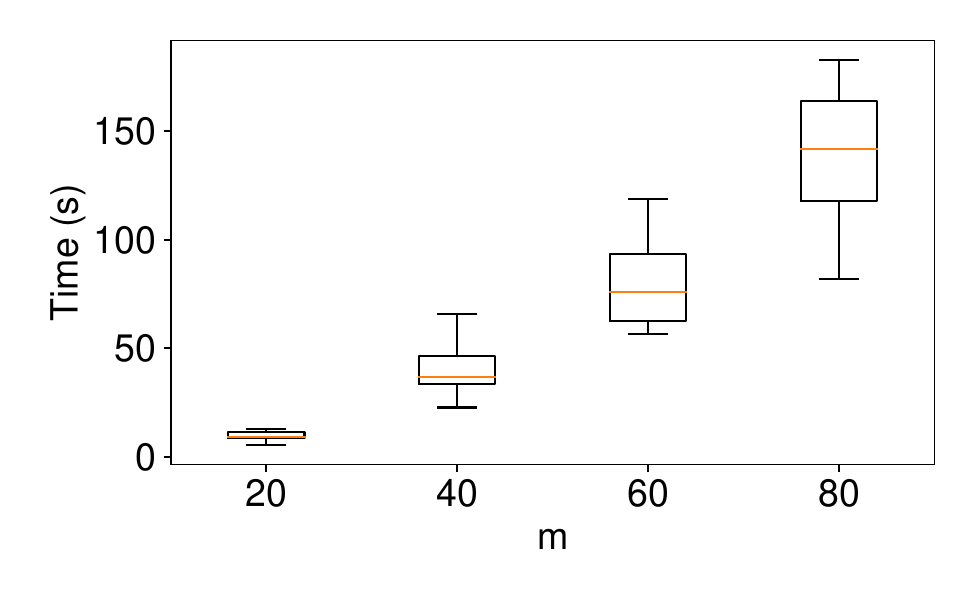}  
\caption{Computational time.}
\label{fig:cc}
\end{figure}

\begin{figure}[!t]
\begin{subfigure}{.245\textwidth}
  \centering
  % include first image
  \includegraphics[width=\linewidth]{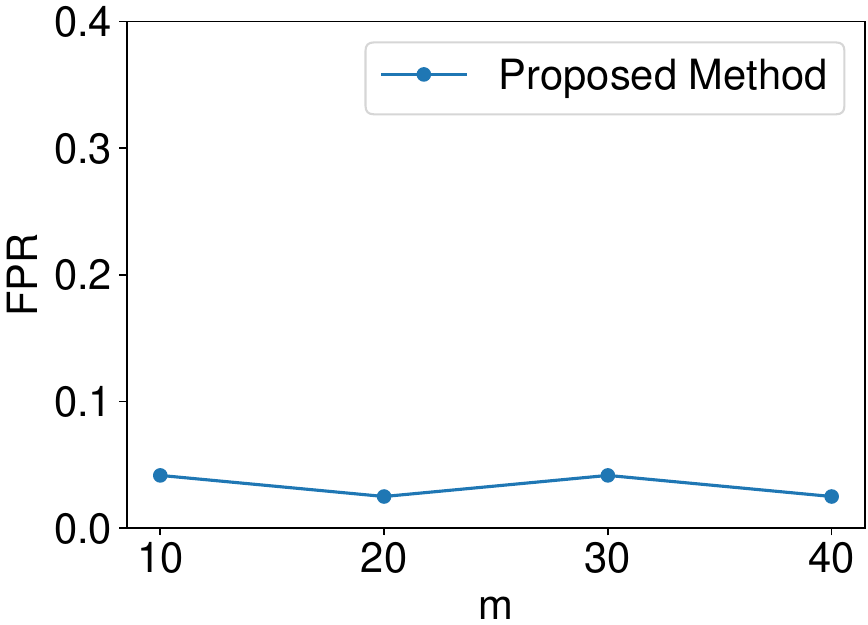}  
  \caption{Laplace distribution}
\end{subfigure}
\begin{subfigure}{.245\textwidth}
  \centering
  \includegraphics[width=\linewidth]{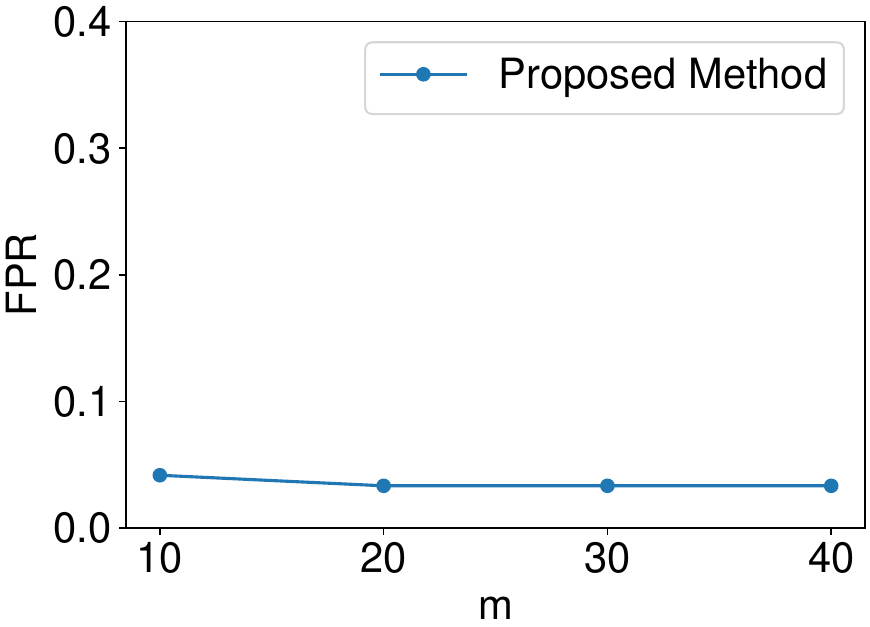}  
  \caption{Skew normal distribution}
\end{subfigure}
\begin{subfigure}{.245\textwidth}
  \centering
  \includegraphics[width=\linewidth]{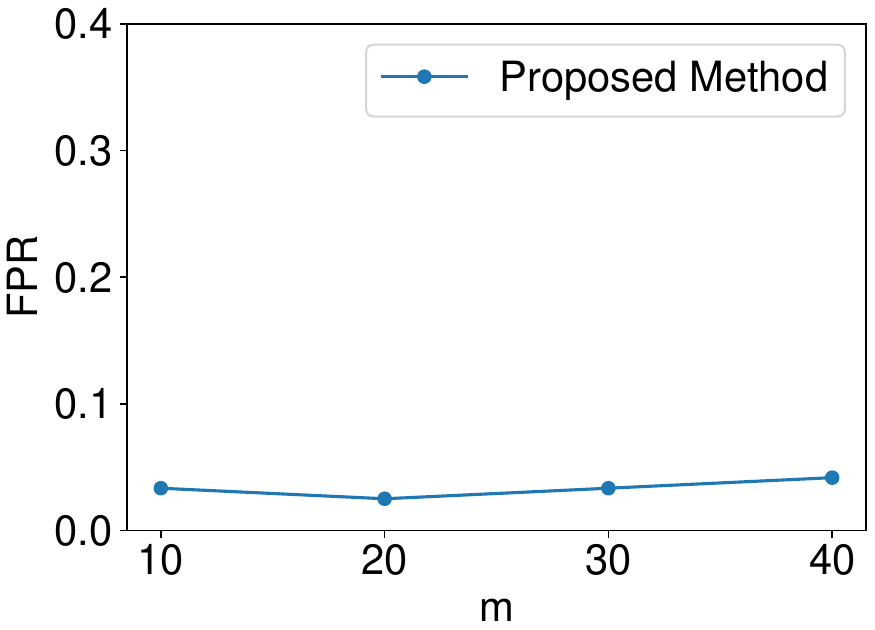}  
  \caption{$t_{20}$ distribution}
\end{subfigure}
\begin{subfigure}{.245\textwidth}
  \centering
  \includegraphics[width=\linewidth]{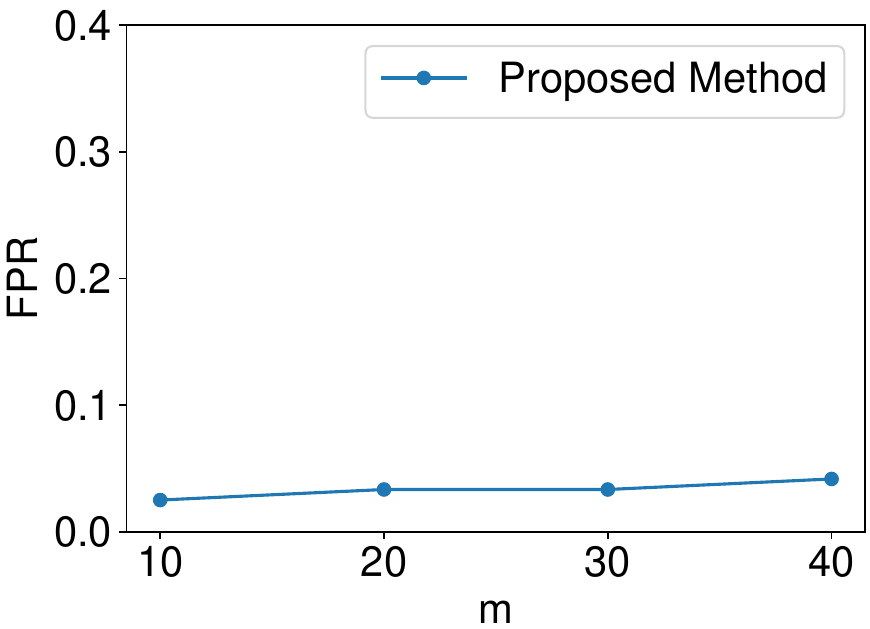}  
  \caption{Estimated variance}
\end{subfigure}
\caption{ The robustness of the proposed method in terms of the FPR control.}
\label{fig:robustness_1}
\end{figure}

\begin{figure}[!t]
\begin{subfigure}{.245\textwidth}
  \centering
  % include first image
  \includegraphics[width=\linewidth]{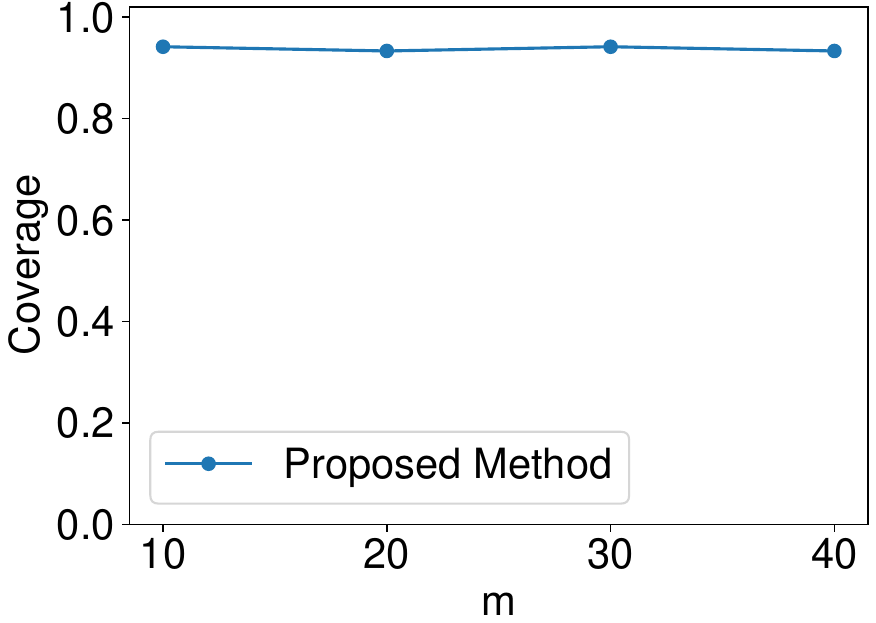}  
  \caption{Laplace distribution}
\end{subfigure}
\begin{subfigure}{.245\textwidth}
  \centering
  \includegraphics[width=\linewidth]{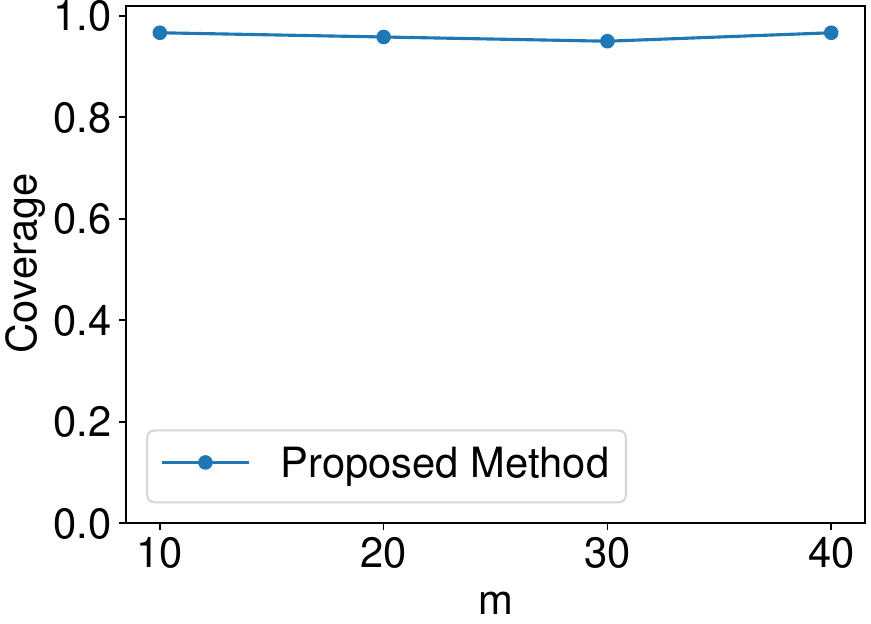}  
  \caption{Skew normal distribution}
\end{subfigure}
\begin{subfigure}{.245\textwidth}
  \centering
  \includegraphics[width=\linewidth]{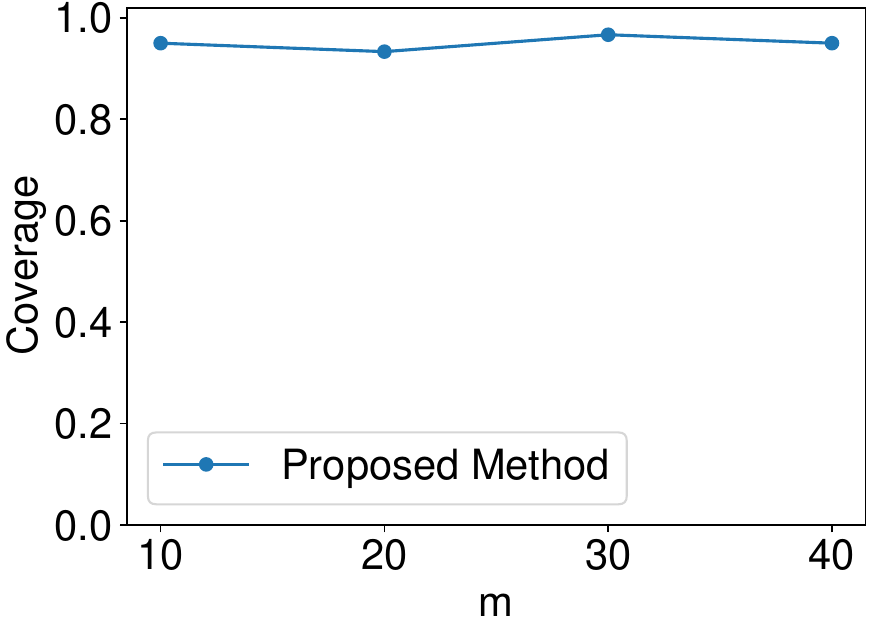}  
  \caption{$t_{20}$ distribution}
\end{subfigure}
\begin{subfigure}{.245\textwidth}
  \centering
  \includegraphics[width=\linewidth]{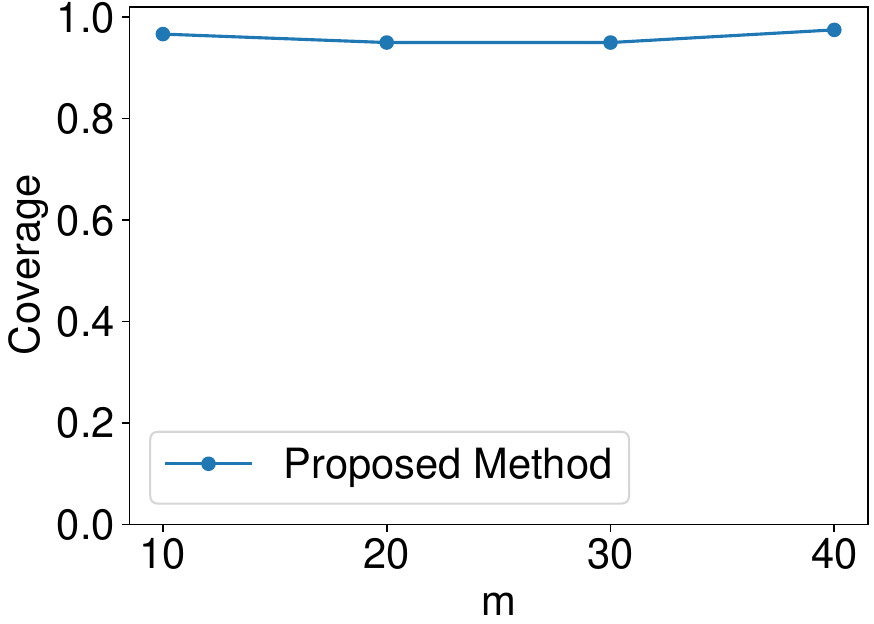}  
  \caption{Estimated variance}
\end{subfigure}
\caption{ The robustness of the proposed method in terms of the CI coverage guarantee.}
\label{fig:robustness_2}
\end{figure}

%\section{Details for Real-data Experiments} \label{appendix:details_real_data}

\subsection{Details on Real-data Experiments} \label{appendix:details_real_data}

In the first problem setting, we consider a two-class classification problem for heart-beat signals where the signals were generated by a data generator tool called 
NeuroKit2 \cite{Makowski2021neurokit}.
In the second setting, we used six real datasets that are available at UCR Time Series Classification Repository and UCI Machine Learning Repository: 
Italy Power Demand (Class $\tt{C1}$: days from Oct to March, Class $\tt{C2}$: days from April to September), 
Melbourne Pedestrian (Class $\tt{C1}$: Bourke Street Mall, Class $\tt{C2}$: Southern Cross Station),
Smooth Subspace (Class $\tt{C1}$: smooth subspace spanning  from time stamp 1 to 5, Class $\tt{C2}$: smooth subspace spanning  from time stamp 11 to 15),
EEG Eye State (Class $\tt{C1}$: eye-open, Class $\tt{C2}$:  eye-closed),
China Town (Class $\tt{C1}$: weekdays, Class $\tt{C2}$:  weekends),
and
Finger Movement (Class $\tt{C1}$: left, Class $\tt{C2}$: right).
These datasets are taken from various application domains and commonly used as the benchmark datasets in time-series analysis.

\subsection{Derivation of the SI-DTW-oc method}\label{appendix:si_dtw_oc}

This is our first idea of introducing conditional SI for time series similarity using DTW by additionally conditioning on all the operations of the DTW algorithm itself to make the problem tractable.
Then, since the selection event of SI-DTW-oc is simply represented as a single polytope in the data space, we can apply the method in the seminal conditional SI paper \cite{lee2016exact} to compute the over-conditioning $p$-value.
However, such an over-conditioning leads to a loss of statistical power \cite{lee2016exact, fithian2014optimal}, i.e., low TPR. 

\textbf{Notation.}
We denote $\cD^{\rm oc}$ as the over-conditioning data space in SI-DTW-oc.
The difference between $\cD$ in \eq{eq:conditional_data_space} and $\cD^{\rm oc}$ is that the latter is characterized with additional constraints on all the operations of the DTW algorithm.
For two time series with lengths $i \in [n]$ and $j \in [m]$, a set of all possible alignment matrices is defined as $\cM_{i, j}$.
Given $\bm X \in \RR^n$ and $\bm Y \in \RR^m$, the loss between theirs sub-sequence $\bm X_{1:i}$ and $\bm Y_{1:j}$ with $M \in \cM_{i, j}$ is written as 
\begin{align*}
	L_{i, j} (\bm X, \bm Y, M) = \Big \langle
		M, C\big (\bm X_{1:i}, \bm Y_{1:j} \big )
	\Big \rangle 
\end{align*}
Then, the DTW distance and the optimal alignment matrix between $\bm X_{1:i}$ and $\bm Y_{1:j}$ are respectively written as
\begin{align*}
	\hat{L}_{i, j} (\bm X, \bm Y) &= \min \limits_{M \in \cM_{i, j}} ~ L_{i, j} (\bm X, \bm Y, M) \\ 
	\hat{M}_{i, j} (\bm X, \bm Y) &= \argmin \limits_{M \in \cM_{i, j}} ~ L_{i, j} (\bm X, \bm Y, M).
\end{align*}

\textbf{Characterization of the over-conditioning conditional data space $\cD^{\rm oc}$.} Since the inference is conducted with additional conditions on all steps of the DTW, the conditional data space $\cD^{\rm oc}$ is written as 
\begin{align} \label{eq:oc_conditional_data_space}
\cD^{\rm oc} = 
\left \{ 
	{\bm X \choose \bm Y} %\in \RR^{n + m}
	 \mid 
	\begin{array}{l}
	\bigcap \limits_{i = 1}^n
	\bigcap \limits_{j = 1}^m
	\hat{M}_{i, j}(\bm X, \bm Y) = \hat{M}_{i, j}^{\rm obs}, \\ 
	\cS(\bm X, \bm Y) = \hat{\bm s}^{\rm obs}, 
	~ \cQ(\bm X, \bm Y) = \hat{\bm q}^{\rm obs}
	\end{array}  
\right \},
\end{align}
where $\hat{M}_{i, j}^{\rm obs} = \hat{M}_{i, j}(\bm X^{\rm obs}, \bm Y^{\rm obs})$.
The characterization of the third condition $\cQ(\bm X, \bm Y) = \hat{\bm q}^{\rm obs}$ is a line in the data space as presented in Lemma \ref{lemma:data_line}.
The characterization of the second condition $\cS(\bm X, \bm Y) = \hat{\bm s}^{\rm obs}$ is the same as Lemma \ref{lemma:cZ_2}.
Therefore, the remaining task is to characterize the region in which the data satisfies the first condition.

For each value of $i \in [n]$ and $j \in [m]$, $\hat{M}_{i, j}(\bm X, \bm Y) = \hat{M}_{i, j}^{\rm obs}$ if and only if 
\begin{align}
	\min \limits_{M \in \cM_{i, j}} ~ L_{i, j} (\bm X, \bm Y, M) &=  L_{i, j} (\bm X^{\rm obs}, \bm Y^{\rm obs}, M_{i, j}^{\rm obs}) \\ 
	\Leftrightarrow \quad \quad \quad \quad \quad 
	\hat{L}_{i, j}(\bm X, \bm Y) &=  L_{i, j} (\bm X^{\rm obs}, \bm Y^{\rm obs}, M_{i, j}^{\rm obs}). \label{eq:app_1}
\end{align}
Based on the recursive structure of DTW, we have 
\begin{align}
	\hat{L}_{i, j}(\bm X, \bm Y) = C_{ij}(\bm X, \bm Y) + \min \left \{
	 \begin{array}{l}
	 \hat{L}_{i - 1, j} (\bm X, \bm Y), \\ 
	 \hat{L}_{i, j - 1} (\bm X, \bm Y), \\ 
	 \hat{L}_{i - 1, j - 1} (\bm X, \bm Y)
	 \end{array}
	 \right \}. \label{eq:app_2}
\end{align}
Combining \eq{eq:app_1} and \eq{eq:app_2}, we have the following inequalities 
\begin{align} \label{eq:app_3}
\begin{aligned}
	L_{i, j} (\bm X^{\rm obs}, \bm Y^{\rm obs}, M_{i, j}^{\rm obs}) &\leq C_{ij}(\bm X, \bm Y) + \hat{L}_{i - 1, j} (\bm X, \bm Y), \\ 
	L_{i, j} (\bm X^{\rm obs}, \bm Y^{\rm obs}, M_{i, j}^{\rm obs}) &\leq C_{ij}(\bm X, \bm Y) + \hat{L}_{i, j - 1} (\bm X, \bm Y), \\
	L_{i, j} (\bm X^{\rm obs}, \bm Y^{\rm obs}, M_{i, j}^{\rm obs}) &\leq C_{ij}(\bm X, \bm Y) + \hat{L}_{i - 1, j - 1} (\bm X, \bm Y).
\end{aligned}
\end{align}
Since the loss function is in the quadratic form, \eq{eq:app_3} can be easily written in the form of
\begin{align*}
	(\bm X ~ \bm Y)^\top A_{i, j}^{(1)} (\bm X ~ \bm Y) \leq 0, \\ 
	(\bm X ~ \bm Y)^\top A_{i, j}^{(2)} (\bm X ~ \bm Y) \leq 0, \\ 
	(\bm X ~ \bm Y)^\top A_{i, j}^{(3)} (\bm X ~ \bm Y) \leq 0.
\end{align*}
where the matrices $A_{i, j}^{(1)}$, $A_{i, j}^{(2)}$ and $A_{i, j}^{(3)}$ depend on $i$ and $j$.
It suggests that the conditional data space in \eq{eq:oc_conditional_data_space} can be finally characterized as
\begin{align*}
\cD^{\rm oc} = 
\left \{ 
	{\bm X \choose \bm Y} 
	 \mid 
	\begin{array}{l}
	\bigcap \limits_{i = 1}^n
	\bigcap \limits_{j = 1}^m
	\bigcap \limits_{k = 1}^3
	(\bm X ~ \bm Y)^\top A_{i, j}^{(k)} (\bm X ~ \bm Y) \leq 0 , \\ 
	\cS(\bm X, \bm Y) = \hat{\bm s}^{\rm obs}, 
	~ \cQ(\bm X, \bm Y) = \hat{\bm q}^{\rm obs}
	\end{array}  
\right \}.
\end{align*}
Now that the conditional data space $\cD^{\rm oc}$ is identified, we can easily compute the truncation region and calculate the over-conditioning selective $p$-value.

%% file: ms.bbl
\begin{thebibliography}{10}

\bibitem{aggarwal2017outlier}
C.~C. Aggarwal.
\newblock {\em Outlier Analysis}.
\newblock Springer, 2017.

\bibitem{chen2019valid}
S.~Chen and J.~Bien.
\newblock Valid inference corrected for outlier removal.
\newblock {\em Journal of Computational and Graphical Statistics}, pages 1--12,
  2019.

\bibitem{choi2017selecting}
Y.~Choi, J.~Taylor, and R.~Tibshirani.
\newblock Selecting the number of principal components: Estimation of the true
  rank of a noisy matrix.
\newblock {\em The Annals of Statistics}, 45(6):2590--2617, 2017.

\bibitem{cuturi2017soft}
M.~Cuturi and M.~Blondel.
\newblock Soft-dtw: a differentiable loss function for time-series.
\newblock In {\em International conference on machine learning}, pages
  894--903. PMLR, 2017.

\bibitem{duy2020quantifying}
V.~N.~L. Duy, S.~Iwazaki, and I.~Takeuchi.
\newblock Quantifying statistical significance of neural network-based image
  segmentation by selective inference.
\newblock {\em Advances in Neural Information Processing Systems},
  35:31627--31639, 2022.

\bibitem{le2021parametric}
V.~N.~L. Duy and I.~Takeuchi.
\newblock Parametric programming approach for more powerful and general lasso
  selective inference.
\newblock In {\em International Conference on Artificial Intelligence and
  Statistics}, pages 901--909. PMLR, 2021.

\bibitem{duy2021more}
V.~N.~L. Duy and I.~Takeuchi.
\newblock More powerful conditional selective inference for generalized lasso
  by parametric programming.
\newblock {\em The Journal of Machine Learning Research}, 23(1):13544--13580,
  2022.

\bibitem{duy2021exact}
V.~N.~L. Duy and I.~Takeuchi.
\newblock Exact statistical inference for the wasserstein distance by selective
  inference.
\newblock {\em Annals of the Institute of Statistical Mathematics},
  75(1):127--157, 2023.

\bibitem{duy2020computing}
V.~N.~L. Duy, H.~Toda, R.~Sugiyama, and I.~Takeuchi.
\newblock Computing valid p-value for optimal changepoint by selective
  inference using dynamic programming.
\newblock In {\em Advances in Neural Information Processing Systems}, 2020.

\bibitem{fithian2014optimal}
W.~Fithian, D.~Sun, and J.~Taylor.
\newblock Optimal inference after model selection.
\newblock {\em arXiv preprint arXiv:1410.2597}, 2014.

\bibitem{hastie2004entire}
T.~Hastie, S.~Rosset, R.~Tibshirani, and J.~Zhu.
\newblock The entire regularization path for the support vector machine.
\newblock {\em Journal of Machine Learning Research}, 5(Oct):1391--1415, 2004.

\bibitem{hyun2018post}
S.~Hyun, K.~Lin, M.~G'Sell, and R.~J. Tibshirani.
\newblock Post-selection inference for changepoint detection algorithms with
  application to copy number variation data.
\newblock {\em arXiv preprint arXiv:1812.03644}, 2018.

\bibitem{keogh2001derivative}
E.~J. Keogh and M.~J. Pazzani.
\newblock Derivative dynamic time warping.
\newblock In {\em Proceedings of the 2001 SIAM international conference on data
  mining}, pages 1--11. SIAM, 2001.

\bibitem{lee2016exact}
J.~D. Lee, D.~L. Sun, Y.~Sun, and J.~E. Taylor.
\newblock Exact post-selection inference, with application to the lasso.
\newblock {\em The Annals of Statistics}, 44(3):907--927, 2016.

\bibitem{liu2018more}
K.~Liu, J.~Markovic, and R.~Tibshirani.
\newblock More powerful post-selection inference, with application to the
  lasso.
\newblock {\em arXiv preprint arXiv:1801.09037}, 2018.

\bibitem{loftus2014significance}
J.~R. Loftus and J.~E. Taylor.
\newblock A significance test for forward stepwise model selection.
\newblock {\em arXiv preprint arXiv:1405.3920}, 2014.

\bibitem{loftus2015selective}
J.~R. Loftus and J.~E. Taylor.
\newblock Selective inference in regression models with groups of variables.
\newblock {\em arXiv preprint arXiv:1511.01478}, 2015.

\bibitem{mairal2012complexity}
J.~Mairal and B.~Yu.
\newblock Complexity analysis of the lasso regularization path.
\newblock {\em arXiv preprint arXiv:1205.0079}, 2012.

\bibitem{Makowski2021neurokit}
D.~Makowski, T.~Pham, Z.~J. Lau, J.~C. Brammer, F.~Lespinasse, H.~Pham,
  C.~Sch{\"o}lzel, and S.~H.~A. Chen.
\newblock Neurokit2: A python toolbox for neurophysiological signal processing.
\newblock {\em Behavior Research Methods}, Feb 2021.

\bibitem{muller2007dynamic}
M.~M{\"u}ller.
\newblock Dynamic time warping.
\newblock {\em Information retrieval for music and motion}, pages 69--84, 2007.

\bibitem{panigrahi2016bayesian}
S.~Panigrahi, J.~Taylor, and A.~Weinstein.
\newblock Bayesian post-selection inference in the linear model.
\newblock {\em arXiv preprint arXiv:1605.08824}, 28, 2016.

\bibitem{park2007l1}
M.~Y. Park and T.~Hastie.
\newblock L1-regularization path algorithm for generalized linear models.
\newblock {\em Journal of the Royal Statistical Society: Series B (Statistical
  Methodology)}, 69(4):659--677, 2007.

\bibitem{sakoe1978dynamic}
H.~Sakoe and S.~Chiba.
\newblock Dynamic programming algorithm optimization for spoken word
  recognition.
\newblock {\em IEEE transactions on acoustics, speech, and signal processing},
  26(1):43--49, 1978.

\bibitem{sugiyama2021more}
K.~Sugiyama, V.~N. Le~Duy, and I.~Takeuchi.
\newblock More powerful and general selective inference for stepwise feature
  selection using homotopy method.
\newblock In {\em International Conference on Machine Learning}, pages
  9891--9901. PMLR, 2021.

\bibitem{sugiyama2021valid}
R.~Sugiyama, H.~Toda, V.~N.~L. Duy, Y.~Inatsu, and I.~Takeuchi.
\newblock Valid and exact statistical inference for multi-dimensional multiple
  change-points by selective inference.
\newblock {\em arXiv preprint arXiv:2110.08989}, 2021.

\bibitem{tanizaki2020computing}
K.~Tanizaki, N.~Hashimoto, Y.~Inatsu, H.~Hontani, and I.~Takeuchi.
\newblock Computing valid p-values for image segmentation by selective
  inference.
\newblock In {\em Proceedings of the IEEE/CVF Conference on Computer Vision and
  Pattern Recognition}, pages 9553--9562, 2020.

\bibitem{tian2018selective}
X.~Tian and J.~Taylor.
\newblock Selective inference with a randomized response.
\newblock {\em The Annals of Statistics}, 46(2):679--710, 2018.

\bibitem{tibshirani2016exact}
R.~J. Tibshirani, J.~Taylor, R.~Lockhart, and R.~Tibshirani.
\newblock Exact post-selection inference for sequential regression procedures.
\newblock {\em Journal of the American Statistical Association},
  111(514):600--620, 2016.

\bibitem{tsukurimichi2021conditional}
T.~Tsukurimichi, Y.~Inatsu, V.~N.~L. Duy, and I.~Takeuchi.
\newblock Conditional selective inference for robust regression and outlier
  detection using piecewise-linear homotopy continuation.
\newblock {\em arXiv preprint arXiv:2104.10840}, 2021.

\bibitem{yang2016selective}
F.~Yang, R.~F. Barber, P.~Jain, and J.~Lafferty.
\newblock Selective inference for group-sparse linear models.
\newblock In {\em Advances in Neural Information Processing Systems}, pages
  2469--2477, 2016.

\end{thebibliography}
